\documentclass{article}

\usepackage{amsmath}
\usepackage{amssymb,amsthm}
\usepackage{anysize}
\usepackage{enumerate}
\usepackage{bbold}
\usepackage{bbm,bm}

\usepackage[onehalfspacing]{setspace} 
\usepackage{authblk}

\input{macros_nicole}

\definecolor{cite_color}{rgb}{0, 0.3, 0.6}

\usepackage[colorlinks=true,linkcolor=cite_color, citecolor=cite_color]{hyperref}
\usepackage[round]{natbib}
\newcommand{\mycite}[1]{\cite{#1}} 

\newcommand{\lip}{{\rm Lip}}

\def\NN{\mathbb N}

\def\RR{\mathbb R}

\def\EE{\mathbb E}

\def\HH{\mathcal H}
\def\FF{\mathcal F}

\def\xx{\mathbf{x}}
\def\yy{\mathbf{y}}
\def\zz{\mathbf{z}}

\def\eld{\ell'}

\newcommand{\ka} {\kappa_0}
\newcommand{\kao} {\kappa_1}

\newcommand{\la} {\lambda}
\newcommand{\sx}{\hat S}
\newcommand{\ip}{S}
\newcommand{\lx}{\hat L}
\newcommand{\lp}{L}

\newcommand{\bx}{\hat B}
\newcommand{\bp}{B}

\newcommand{\tx}{\hat T}
\newcommand{\tp}{T}

\newcommand{\ft}{f_t}

\newcommand{\gt}{g_t}
\newcommand{\et}{e_t}
\newcommand{\ek}{e_t}
\newcommand{\zt}{z_t}
\newcommand{\fto}{f_{t+1}}
\newcommand{\gto}{g_{t+1}}
\newcommand{\zto}{z_{t+1}}
\newcommand{\eto}{e_{t+1}}
\newcommand{\VE}{\bm{\epsilon}}

\newcommand{\LL}{{L}^2}

\newcommand{\fp}{f^\dagger}
\newcommand{\gp}{g^\dagger}

\newcommand{\fmin}{f^\dagger}

\newcommand{\sbrac}[1]{\left[#1\right]}

\newcommand{\scalar}[3]{\langle{ #1},{#2} \rangle_{#3}}
\newcommand{\abs}[1]{\left\lvert #1 \right\rvert}

\title{Gradient-Based Non-Linear Inverse Learning}

\author[1]{Abhishake}
\author[2]{Nicole M\"ucke}
\author[3]{Tapio Helin}
\affil[1,3]{Department of Computational Engineering, School of Engineering Science, LUT University, Yliopistonkatu 34, 53850 Lappeenranta, Finland}
\affil[2]{Institut f\"ur Mathematische Stochastik, Technische Universit\"at Braunschweig, Universit\"atsplatz 2, 38106 Braunschweig, Germany}

\begin{document}

\maketitle

\begin{abstract}
We study statistical inverse learning in the context of nonlinear inverse problems under random design. Specifically, we address a class of nonlinear problems by employing gradient descent (GD) and stochastic gradient descent (SGD) with mini-batching, both using constant step sizes. Our analysis derives convergence rates for both algorithms under classical a priori assumptions on the smoothness of the target function. These assumptions are expressed in terms of the integral operator associated with the tangent kernel, as well as through a bound on the effective dimension. Additionally, we establish stopping times that yield minimax-optimal convergence rates within the classical reproducing kernel Hilbert space (RKHS) framework. These results demonstrate the efficacy of GD and SGD in achieving optimal rates for nonlinear inverse problems in random design.
\end{abstract}




\section{Introduction}

Nonlinear inverse problems arise across a wide range of scientific, engineering, and medical contexts, where the goal is to determine underlying parameters, structures, or causes from indirect, incomplete, or noisy measurements. Examples include modeling physical phenomena governed by partial differential equations, electrical impedance tomography, and detecting cracks or flaws in materials using ultrasonic, thermal, or electromagnetic signals, which often involve nonlinear wave propagation.

\vspace{0.2cm} 

Statistical inverse learning in random design is 
a subfield of statistical learning that focuses on solving inverse 
problems where the design (input) points are randomly sampled from some unknown distribution. More specifically, the observations 
follow a model 
\[ Y = (A\fp)(X) + \epsilon \;, \]
where $A:\cH_1 \to \cH_2$ is a linear or non-linear operator that maps between two Hilbert spaces $\cH_1$ and $\cH_2$. Here, $(X,Y)$ 
is randomly sampled from an unknown probability distribution, rather than being fixed or chosen in a deterministic way.
The noise $\epsilon$ is random as well and centered, conditioned on the design points. 
The (potentially ill-posed) inverse problem involves determining the unknown function $\fp$. The randomness can 
introduce additional complexity into the inverse problem, as the statistical properties of the random 
inputs need to be taken into account.

The deep connection between statistical learning and linear inverse problems has been made evident in a variety of pioneering works, 
starting with \mycite{de2006discretization, de2005learning}, where the benefit of regularization techniques from inverse problems \mycite{Engl96}, 
such as Tikhonov regularization, in machine learning is highlighted, helping to control overfitting. 
These techniques add additional information or constraints to the learning problem, guiding the solution towards a 
more reasonable and stable outcome. The statistical properties of those algorithms in the context of 
prediction and kernel regression have subsequently been analyzed in many works, 
e.g. \mycite{BauPerRos07, gerfo2008spectral, steinwart2009optimal, mendelson2010, caponnetto2007optimal, blanchard2020kernel, lin2020optimal}, 
to mention just a few.

Solving linear inverse learning problems with general spectral regularization methods is investigated in \mycite{BlaMuc16}, 
further developing the original ideas from \mycite{de2005learning} from prediction to reconstruction and providing minimax optimal bounds. 
Regularization in Hilbert scales is analyzed in \mycite{Rastogi23}, extending the classical analysis to the inverse learning setting.  
These ideas have been further advanced for regularization by projection \mycite{Helin22}, showing that 
coupling the number of random point evaluations with the choice of projection dimension, one can
derive probabilistic convergence rates for the reconstruction error of the maximum likelihood estimator. 
In \mycite{bubba2023convex}, Tikhonov regularization for inverse learning is extended to general convex regularization 
with $p$-homogeneous penalty functionals. Concentration rates of
the regularized solution to the ground truth, measured in the symmetric Bregman distance
induced by the penalty functional, are derived. 



\vspace{0.2cm}

While the statistical analysis of regularization approaches for linear inverse problems is mostly well understood, 
the analysis for non-linear problems in the context of inverse learning is rather scarce. The first attempts at 
analyzing Tikhonov regularization in the non-linear case are made in \mycite{Rastogi20}. A theoretical analysis is 
developed, and optimal rates of convergence are discussed. This analysis is then extended in \mycite{Rastogi20a, Rastogi24} 
to the settings of Tikhonov regularization in Hilbert scales and with oversmoothing penalty. However, the drawback of 
these approaches lies in their lack of practical implementability. The minimizer of the Tikhonov functional has no explicit solution in the non-linear case and thus has to be approximated, preferably with an iterative method that is efficient to implement. 

\vspace{0.2cm} 

In the context of nonlinear inverse problems with a fixed design, where the data are considered deterministic, iterative methods such as Landweber iteration, Newton-type methods, or multilevel methods are analyzed with respect to their convergence properties under specific assumptions on the nonlinear operator to be inverted (e.g., Fréchet differentiability); see, for instance, \mycite{Hanke95} and \mycite{kaltenbacher08}. However, the fixed design setting fails to properly incorporate the random nature of the data. A random design setting, on the other hand, allows for the incorporation of the stochastic nature of data, enabling the modeling of uncertainty and variability inherent in real-world measurements. This can lead to more robust and generalizable solutions, as it facilitates statistical analysis, regularization techniques, and error estimation that account for noise and randomness. For a recent overview and more details, we refer the reader to \mycite{helin2023statistical}.





\vspace{0.3cm}

\paragraph{Our Contribution.} We follow the framework of statistical nonlinear inverse learning introduced in \mycite{Rastogi20}. Extending this approach, we present gradient descent (GD) and stochastic gradient descent (SGD), both with constant step sizes and mini-batching for SGD, providing gradient-based iterative algorithms that are numerically implementable. To efficiently achieve this, we introduce the \emph{tangent kernel}, derived from the linearization of the nonlinear operator to be inverted. This kernel is used to formulate \emph{a priori} smoothness assumptions on the objective function in terms of the associated integral operator.

\vspace{0.2cm}
For GD, we demonstrate that the algorithm is indeed a descent method with high probability, up to a specific threshold in the number of iterations. This threshold serves as the optimal stopping time, ensuring a fast convergence rate.

\vspace{0.2cm} 

For mini-batch SGD, we derive convergence rates under different assumptions regarding the step size, stopping time, and mini-batch size, analyzing how these choices influence the algorithm's performance and providing deeper insights into the trade-offs inherent in practical implementations. Our results explicitly quantify the convergence rate in terms of the number of passes over the training data. 

\vspace{0.2cm} 

All our derived convergence rates are known to be minimax optimal in the classical framework of nonparametric regression in reproducing kernel Hilbert spaces and match those of linear inverse problems in random design.

\vspace{0.2cm}

\paragraph{Organization.} Our paper is structured as follows: {\bf Section \ref{sec:intro-setting}}  introduces the framework of statistical inverse learning under random design. We define the tangent kernel, arising from a linearization of the problem. Additionally, the algorithms analyzed in this paper, specifically gradient descent (GD) and stochastic gradient descent (SGD) with mini-batching, are presented. In {\bf Section \ref{sec:main-results}}, we outline several key assumptions necessary for the analysis. We then present the main convergence results for both gradient descent and stochastic gradient descent with mini-batching. These results establish convergence rates under standard a priori smoothness assumptions on the target function, expressed in terms of the integral operator associated with the tangent kernel and effective dimension bounds. {\bf Section \ref{sec:discussion}} provides a comprehensive discussion of the derived results. We compare our findings to existing literature, highlighting similarities, differences, and the implications of our approach. All proofs supporting the main analysis are provided in the {\bf  Appendices \ref{app:prelim}, \ref{app:A-proofs-GD}, \ref{app:proofs-SGD}}, and {\bf \ref{Sec:Prob.bound}}.

\vspace{0.2cm}

\paragraph{Notation.} 
For a Hilbert space $\cH$ we denote its norm by $||\cdot||_{\cH}$ and inner product $\inner{\cdot,\cdot}_{\cH}$.   
With $\cL(\cH_1, \cH_2)$ we denote the set of all bounded linear operators between the separable Hilbert spaces $\cH_1$ to $\cH_2$. If $\cH_1 = \cH_2 =\cH$ we write $\cL(\cH_1, \cH_2)=\cL(\cH)$. 

\vspace{0.2cm}

For $A \in \cL(\cH)$ we let $A^*$ denote the adjoint operator. By $\cD(A)$ we denote the domain of an operator $A$.  While $||A||$ is the operator norm, we write $||A||_{HS}$ for the Hilbert Schmidt norm.

\vspace{0.2cm} 
For any positive integer $n \in \mbn$ we denote $[n]:=\{1, \ldots ,n\}$. Given two sequences $(a_n)_{n \in \mbn}$, $(b_n)_{n \in \mbn}$, we write $a_n \lesssim b_n$ if there is a constant $c \in \mbr_+$ such that $a_n \leq c b_n$, for $n$ sufficiently large. Similarly,  $a_n \gtrsim b_n$ if there is a constant $c \in \mbr_+$ such that $a_n \geq c b_n$, for $n$ sufficiently large. Finally, $a_n \simeq b_n$ means $a_n \lesssim b_n$ and $a_n \gtrsim b_n$.


\section{Statistical Non-Linear Inverse Learning}
\label{sec:intro-setting}

\subsection{Setting}

{\bf Statistical Model and Noise Assumption.}  Statistical inverse learning involves inferring an 
unknown function \( \fp \) using statistical methods applied to point evaluations of a related 
function \( \gp \). These evaluations may be sparse or noisy, and the relation 
between \( \fp \) and \( \gp \) is established through the model
\begin{equation}
\label{model}
  A(\fmin) = \gp , \quad \text{for } \fmin \in \cH_1 \text{ and } \gp \in \cH_2,  
\end{equation}
where \( A \) is a nonlinear operator between real separable Hilbert spaces \( \cH_1 \) and  \( \cH_2 \).

\vspace{0.2cm}

In \emph{supervised inverse learning}, it is assumed that the image space $\cH_2$ consists of functions $g: \cX \to \cY$. 
Here, we assume the input space $\cX$ 
to be a Polish space and the output space $\cY$ to be a separable Hilbert space. 
Moreover, we assume that function evaluation is continuous. Consequently, for any 
$x \in \cX$, the values $g(x)=[A(f)](x)$ are well-defined elements in $\cY$.

\vspace{0.2cm}

We focus on random design regression in inverse problems, where the data 
$((x_1, y_1), \ldots, (x_n, y_n)) \in (\cX \times \cY)^n$ emerge due to random observations 
\begin{equation}
\label{model2}
  y_j = \gp(x_j) + \epsilon_j \;. 
\end{equation}
Here, the outputs $y_j\in \cY$ are noisy observations 
of $\gp(x_j)$ at random inputs $x_j \in \cX$, $j=1,\ldots,n$. 
The noise  variables $\epsilon_j$ are independent and centered with $\mbe[\epsilon_j | X= x_j] = 0$. 

\vspace{0.2cm}

The data $(x_1, y_1), \ldots, (x_n, y_n)$ are realizations of a pair of random variables $(X, Y) \sim \rho$, with $\rho$ 
being an unknown probability measure on the Borel $\sigma$-field of $\cX \times \cY$. 
We denote by $\nu$ the marginal distribution on $\cX$ and by $\rho(\cdot|x)$ the conditional 
distribution on $\cY$ with respect to $x \in \cX$, whose existence is assumed. 

\vspace{0.2cm}

The expected square loss (or risk)
\begin{equation}
\label{eq:E}
 \cE(f) = \EE\sbrac{\norm{ [ A(f)](X) - Y }_{\cY}^2} = \int_{\cX \times \cY} \norm{ [ A(f)](x) - y }_{\cY}^2\; d\rho(x,y) 
\end{equation}
quantifies the discrepancy between the predicted values and the true labels, 
with pairs $(x,y)$ that are more likely to be sampled contributing more significantly to 
the overall error. We are interested in providing convergence rates for the reconstruction and prediction 
error for the unknown $\fp$ by minimizing $\cE$. 
The minimizers of the expected risk typically coincide with the classical minimum norm solutions of inverse problems with the difference that
data fidelity is weighted by the design measure, see \mycite{helin2023statistical}.

\vspace{0.3cm}

In what follows, we state our assumptions on the solution $\fp$ of the inverse 
problem \eqref{model} and
the noise in \eqref{model2}. 

\vspace{0.3cm}

\begin{assumption}[True solution $\fp$]
\label{ass:true}
The conditional expectation of $Y$ given $X$ exists a.s. and there exists an $\fp \in \mbox{int}(\cD(A)) \subset \cH_1$ such that 
\[ [A(\fp)] (x) = \int_\cY y \; \rho(dy|x) \;, \quad \nu \mbox{-almost surely} \;. \] 
\end{assumption}

It can be easily observed that $\fp$ is also the minimizer of the expected risk.

\vspace{0.2cm}

We will need two different assumptions on the outputs. For our analysis of stochastic gradient descent, we require the output variable to be bounded in norm: 

\begin{assumption}[Bounded outputs]
\label{ass:bounded}
Assume that $\norm{Y}_{\cY}\leq \widetilde{M}$ almost surely, for some $\widetilde{M}>0$.
\end{assumption}

\vspace{0.2cm}

However, for our analysis of gradient descent, we can relax this assumption to: 
\vspace{0.2cm}

\begin{assumption}[Noise condition]\label{ass:noise}
There exist some constants~$M,\Sigma$ such that for almost all~$x\in X$ and $l\geq 2$,
\begin{equation}\label{noise.cond}
\int_Y \norm{y-A(\fmin)(x)}_{\cY}^l d\rho(y|x) \leq \frac{1}{2} l! \Sigma^2 M^{l-2}.
\end{equation}
\end{assumption}

This assumption is usually referred to as a \emph{Bernstein-type assumption}. It describes the characteristics of the noise influencing the output $y$. 
This condition holds true in multiple scenarios, such as when the noise $\epsilon$ is either 
bounded or follows a sub-Gaussian distribution with a zero mean, independent of $x$ 
(see \mycite{wellner2013weak}). However, note that this does not apply to 
Gaussian white noise in infinite-dimensional spaces.

It is straightforward to observe that Assumption \ref{ass:bounded} implies the inequality \eqref{noise.cond}.
For $l\geq 2$, we have almost surely:
\begin{align*}
\int_Y \norm{y-A(\fmin)(x)}_{\cY}^l d\rho(y|x) \leq  & 2^{l-1} \paren{\int_Y \norm{y}_{\cY}^l d\rho(y|x)+\norm{A(\fp)(x)}_{\cY}^l}\\
\leq  & 2^{l-1} \paren{\widetilde{M}^l+\paren{\ka \norm{A(\fp)}}^l} \leq 2^l M^l\\
\leq & \frac{1}{2} l! \Sigma^2 M^{l-2},
\end{align*}
with $M:=\max\paren{\widetilde{M},\ka \norm{A(\fp)}}$ and $\Sigma: = 2M$. This implies the inequality \eqref{noise.cond}.

\vspace{0.3cm}

{\bf RKHS Structure.}  We endow the space $\cH_2$ with an additional structure that allows us to 
apply kernel methods, making practical solutions amenable. 
We concentrate on Hilbert spaces of vector-valued functions that possess a 
reproducing kernel \mycite{alvarez2012kernels, carmeli2006vector, carmeli2010vector}. 
These spaces have garnered significant 
attention in recent years within machine learning theory, owing to the successful 
application of kernel-based learning methods to complex data.

\begin{assumption}
 \label{ass:kernel} Let~$\cH_2$ be a vector-valued 
reproducing kernel Hilbert space of functions~$f:\cX\to \cY$ corresponding to an operator-valued positive semi-definite kernel~$K:\cX\times \cX\to \mathcal{L}(\cY)$. We define the linear function $K_x: \cY \rightarrow \cH_2: y \mapsto K_xy$,  where~$K_xy:\cX \to \cY:x' \mapsto (K_xy)(x')=K(x',x)y$ for~$x,x'\in \cX$ and~$y\in \cY$. Then,  
\begin{enumerate}[(i)]
  \item~$K_x:\cY\to\cH_2$ is a Hilbert-Schmidt operator for~$x\in \cX$ with
    \[\kappa_0^2:=\sup_{x \in \cX} \norm{K_x}_{HS}^2 = {\sup_{x \in
          \cX}\tr(K_x^*K_x)}<\infty.\]
  \item For~$y,y'\in \cY$, the real-valued 
function~$\varsigma: \cX\times \cX \to \RR:(x,x')\mapsto\scalar{ K_{x}y}{K_{x'}y'}{\cH_2}$ is measurable.
\end{enumerate}
\end{assumption}


\subsection{Linearization and Tangent Kernel}

In order to minimize \eqref{eq:E} we apply gradient methods. This requires at least 
the existence of the derivative of the operator $A$ in a vicinity around $f^\dagger$.

\vspace{0.2cm}

\begin{assumption}[Fr{\'e}chet derivative]
\label{ass:Frechet}
The operator $A$ is Fr{\'e}chet differentiable with derivative at $f$ denoted by $A'(f)$. 
Moreover, $A'(f)$ is bounded in a ball $\cB_d(\fp)$ of radius $d$ centered at $\fp$ with 
\[ \norm{A'(f)} \leq \lip \;, \quad f \in \cB_d(\fp)\cap \cD(A) \;, \]
for some $\lip >0$.
\end{assumption}


\vspace{0.3cm}

In what follows, we show that the gradient-based dynamics is guided by a different kernel, which we call the 
\emph{tangent kernel}, that is related to $A'(\fp)$. To this end, we need an 
additional assumption about point evaluations: 

\vspace{0.3cm}

\begin{assumption}
\label{ass:points}
For any $x \in \cX$, we let $S_x: \cH_2 \to \cY$ denote the point evaluation. 
The evaluation functional 
\begin{align*}
B_x:= S_x \circ A'(\fp) &: \HH_1 \to \cY \\
f &\mapsto B_x(f) =  [A'(\fp)f](x)
\end{align*}
at $x \in X$ is uniformly bounded, i.e., there exists a constant $\kappa_1 < \infty$ such that for
any $x \in X$: 
\[ \norm{B_x(f)}_{\cY} \leq  \kappa_1 \; \norm{ f}_{ \HH_1}. \]
\end{assumption}
Note that we have 
\[ \norm{B_x(f)}_{\cY} \leq \norm{S_x}\cdot \norm{A'(\fp) }\cdot \norm{ f}_{ \HH_1} \leq \kappa_0 \cdot \lip \cdot \norm{ f}_{ \HH_1}. \]
Hence, we may choose $\kappa_1=\kappa_0 \cdot\lip$.

This assumption ensures that the range of $A'(\fp)$ is a vector-valued reproducing kernel Hilbert
space  $\cH_G$, see e.g. \mycite{carmeli2010vector}, with kernel given by 
\begin{equation}
\label{eq:tangent-kernel}
 G_{\fp} = G: \cX \times \cX  \to \cL(\cY)\;, \quad 
G(x, x')= B_x \circ B_{x'}^* \;. 
\end{equation}
Indeed, we have 
\[ B_x^*y = G(\cdot , x)y \]
and for $g \in Ran(A'(\fp))$, the \emph{reproducing property} 
\[ \langle g(x), y \rangle_\cY = \langle g, G(\cdot, x)y \rangle_{\cH_G} \;,\quad x \in \cX, y \in \cY,\]
holds. Moreover, $A'(\fp)$ is a partial isometry between $\cH_1$ and $Ran(A'(\fp))$, see  
\mycite{carmeli2010vector}, Proposition 1.

\vspace{0.2cm}

\begin{definition}[Tangent Kernel]
\label{def:tangent-kernel}
The kernel $G_{\fp}: \cX \times \cX  \to \cL(\cY)$ defined in \eqref{eq:tangent-kernel}
is called the \emph{tangent kernel at $\fp$}. The RKHS generated by $G_{\fp}$ is denoted by $\cH_G$.
\end{definition}

We will investigate how the tangent kernel $G_{\fp}$ at  $\fp$ 
influences the learning properties of gradient-based inverse learning. 
More specifically, if  $T: \cH_1 \to \cH_1$  denotes the associated kernel integral operator, we express \emph{a priori} smoothness assumptions on $ \fp $ in terms of powers of $ T $; see Assumption \ref{ass:source}. This integral operator determines the convergence rates of our gradient-based estimators. Furthermore, all hyperparameters, such as stopping time, step size, and batch size, depend on those powers. We also express our error bounds using the interpolation norms $ ||T^u(\cdot)||_{\cH_1} $ for $ u \in \{0, 1/2\} $; see Section \ref{sec:main-results}. Finally, the feature map $ x \mapsto B_x = S_x \circ A'(\fp) $ enters directly into the algorithms \eqref{Eq:GD-iterate} and \eqref{fk.mini}, enabling direct practical implementation.


\subsection{Gradient Descent and Stochastic Gradient Descent}

In what follows, we define the algorithms we are going to analyze. Gradient descent and stochastic gradient descent 
are based on empirical risk minimization and both perform an implicit regularization. 
We are given a sample $((x_1, y_1), \ldots , (x_n, y_n)) \in (\cX \times \cY)^n$, drawn independently and identically 
from the probability measure $\rho$. We denote by $\by=(y_1, \ldots, y_n)^T \in \cY^n$ the vector of outputs 
and by $\hat S: \cH_2 \to \cY^n$ 
the empirical operator $\hat S g = (g(x_1), \ldots, g(x_n))^T$. Moreover, we define the empirical norm 
$\norm{ \by }^2_n = \frac{1}{n}\sum_{j=1}^n \norm{y_j}^2_{\cY} $. With this notation, the empirical risk is given by 
\[ \widehat \cE(f):= \frac{1}{n}\sum_{j=1}^n\norm{ [A(f)](x_j) - y_j  }_{\cY}^2 
=  \norm{ (\hat S\circ A )(f) - \by }_{n}^2\;. \]

{\bf Gradient Descent.} Given a stepsize $\eta >0$, we initialize at some $g_1 \in \mathcal{D}(A)$. 
The gradient descent iterates are then given by 
\begin{equation}
\label{Eq:GD-iterate}
 \gto = \gt - \eta \paren{\sx \circ A'(\gt)}^*  \left((\sx \circ A)(\gt)-\yy\right)  \;,
\end{equation}
for any $t >1$.

\vspace{0.2cm}

\begin{assumption}[GD Constant Step Size]
\label{ass:eta}
We assume that the stepsize is bounded: $\eta < 1/\kappa^2_1 = 1/(\kappa_0 \lip)^2$\;. 

\end{assumption}

\vspace{0.2cm}

{\bf Stochastic Gradient Descent with mini-batching.} A variant of Stochastic Gradient Descent (SGD) involves processing multiple data points at each 
iteration, a technique known as mini-batching. Given a batch-size $b \in [n]$, the mini-batch SGD update 
rule is defined as follows: We chose a stepsize $\eta >0$ and initialize at some $f_1 \in \mathcal{D}(A)$. 
The consecutive iterates are given by 
\begin{equation}
\label{fk.mini}
\fto= \ft-\frac{\eta}{b}\sum\limits_{i=b(t-1)+1}^{bt}( S_{x_{j_i}}\circ A'(f_t) )^*
\paren{ S_{x_{j_i}}\circ A}(\ft)-y_{j_i}),
\end{equation}
for any $t >1$. Here, $j_1,\ldots,j_b$ are independent and identically distributed random variables, 
each drawn from a uniform distribution over $[n]$. After $t$ iterations, the number of passes 
over the data is $\lceil bt/n\rceil$. Note that with $b=n$, we obtain batch gradient descent.


\section{Main Results}
\label{sec:main-results}

\begin{assumption}
\label{ass:A}
The following conditions hold:
\begin{enumerate}[(i)]
\item There exists an $\alpha \in (0, \frac{1}{2} )$ such that for any $f,\tilde{f}\in \cB_d(\fp)\cap \cD(A)$,
\begin{equation}
\label{eq:lipschitz2}
\norm{A(f)-A(\tilde{f})-A'(\tilde{f})(f-\tilde{f})}_{\LL} \leq \alpha \; \norm{A(f)-A(\tilde{f})}_{\LL} 
\end{equation}
\item There exists a family of uniformly bounded operators $\{R_f\}_f$, $R_f : L^2(\cX, \nu) \to L^2(\cX, \nu)$ 
such that for any 
$f\in \cB_d(\fp) \cap \cD(A)$,
\begin{equation}
\label{eq:Rop}
\ip A'(f)=R_f\ip A'(\fp) \;,
\end{equation}
 with 
\begin{equation}
\label{eq:Rop2}
 \sup_{f \in \cB_d(\fp)\cap \cD(A)}\norm{R_f} \leq \tilde C_R  \;, \quad  
\norm{R_f - I} \leq  C_R  \norm{f-\fp}_{\HH_1}\;,
\end{equation}
for some $\tilde C_R> 0 $ and $C_R > 0$.
\end{enumerate}
\end{assumption}

\vspace{0.2cm}

Note that \eqref{eq:lipschitz2} implies by the triangle inequality the bound 
\begin{equation}
\label{Lem:R.bound}
\frac{1}{1+\alpha} \; \norm{A'(\tilde{f})(f-\tilde{f})}_{\cH_2} 
\leq \norm{A(f)-A(\tilde{f})}_{\cH_2}  
\leq \frac{1}{1-\alpha} \; \norm{A'(\tilde{f})(f-\tilde{f})}_{\cH_2}\;,
\end{equation}
for all $f,\tilde{f}\in \cB_d(\fp)\cap \cD(A)$. Together with Assumption \ref{ass:Frechet} we have 
\begin{equation}
\label{eq:lipschitz3}
\norm{A(f)-A(\tilde{f})}_{\cH_2} \leq \frac{\lip}{1-\alpha}\; \norm{f-\tilde{f}}_{\cH_1} \;. 
\end{equation}

\vspace{0.3cm}

Condition \eqref{eq:lipschitz2} is strong enough to ensure at least local convergence to a solution. 
It also guarantees that all iterates $f_t$ remain in $\cD(A)$ for all $t \in [T_n]$, which makes the 
iteration \eqref{Eq:GD-iterate} well defined. 
Otherwise, it would be necessary to project the iterates onto $\cD(A)$ in each step, see \mycite{Hanke95}.

\vspace{0.3cm}

{\bf Example (Parameter identification).} A prominent example of a non-linear inverse problem comes from parameter identification, see \mycite{Hanke95, kaltenbacher08}. In this problem, we want to estimate the diffusion coefficient \(a\) in
\begin{align}
\label{eq:param-id}
-(a(s)u_s(s))_s &= f(s), \quad s \in (0, 1),\\
u(0) &= 0 = u(1),
\end{align}
where \(f \in L^2[0,1]\). The subscript \(s\) denotes the derivative with respect to \(s\).

In this example, \(A\) is defined as the parameter-to-solution mapping
\begin{align*}
A : \cD(A) := &\{a \in H^1[0, 1] : a(s) \geq \underline{a} > 0\} \to  H^2[0, 1],\\
&a \mapsto  F(a) := u(a),
\end{align*}
where \(u(a)\) is the solution of \eqref{eq:param-id}. One can prove that \(A\) is Fr{\'e}chet-differentiable with
\[
A'(a)h = D(a)^{-1}[(hu_s(a))_s],
\]
\[
A'(a)^*w = -F^{-1}[u_s(a)(A(a)^{-1}w)_s],
\]
where
\[
D(a) : H^2[0, 1] \cap H^1_0[0, 1] \to L^2[0, 1],
\]
\[
u \mapsto D(a)u := -(au_s)_s,
\]
and
\[
F : \cD(F) := \{\psi \in H^2[0, 1] : \psi'(0) = \psi'(1) = 0\} \to L^2[0, 1],
\]
\[
\psi \mapsto B\psi := -\psi'' + \psi.
\]
Note that \(F^{-1}\) is the adjoint of the embedding operator from \(H^1[0, 1]\) into \(L^2[0, 1]\).
It is shown in \mycite{Hanke95, kaltenbacher08} that Assumption \ref{ass:A} is fulfilled for this example. For more examples, we refer to \mycite{Hanke95, kaltenbacher08}.





\vspace{0.2cm}

We move on with introducing more structural assumptions. 

\vspace{0.2cm}

\begin{assumption}[Source Condition]
\label{ass:source}
There exist some $r >0$,  $D>0$ and $g \in \HH_1$ such that
\[ \fp - f_1 = \tp^r g \;,\quad \norm{g}_{\HH_1} \leq D \;.  \]
\end{assumption}

\vspace{0.2cm}

Our H\"older source condition describes the smoothness assumption on the objective function $\fp$ in relation to the RKHS $\cH_G$ generated by the tangent kernel $G_{\fp}$, see Definition \ref{def:tangent-kernel}. It specifies how well the true function $\fp$ can be approximated by the RKHS, as determined by the tangent kernel and the spectral properties of the associated operator $T$. This condition is critical for deriving convergence rates for our gradient-based estimators. 
The parameter $r>0$ quantifies how smooth $\fp$ is relative to the RKHS determined by $G_{\fp}$. Larger values of $r$ indicate higher smoothness, implying that $\fp$ can be well-approximated using the RKHS $\cH_G$. For specific kernels like Sobolev kernels, the H\"older source condition links the smoothness of $\fp$  to classical function spaces, such as Sobolev or Besov spaces. Note, however, that for (non-)linear inverse learning, the kernel $G_{\fp}$ depends on the operator to be inverted, specifically on the Fréchet derivative at $\fp$, see, e.g., \mycite{blanchard2020kernel} for a similar setting in the linear case. 

\vspace{0.2cm}

\begin{assumption}[Polynomial decay condition]
\label{ass:eff.bound}
Let $\lambda >0$. Assume that for some~$0<\nu<1$, there exists some positive constant~$C_\nu>0$ such that
\begin{equation*}
\mathcal{N}(\la):=\tr\left((\lp+\la I)^{-1}\lp\right) \leq C_\nu^2\la^{-\nu}.
\end{equation*}
\end{assumption}

\vspace{0.2cm}

The effective dimension quantifies the complexity of the function space relative to the data and plays a crucial role in understanding the trade-offs between model complexity, estimation accuracy, and computational efficiency. It is defined via spectral calculus 
\[ \cN(\lambda) = \sum_{j=1}^\infty \frac{\sigma_j}{\sigma_j + \lambda} \;, \]
where $\sigma_j$ are the eigenvalues of the integral operator $L$ and  $\lambda >0$ is typically some regularization parameter  that controls the trade-off between bias and variance. The effective dimension measures the number of degrees of freedom of the regression problem that are effectively being used given the regularization level $\lambda >0$. Smaller $\lambda >0$ values allow more eigenvalues $\sigma_j$ to contribute, resulting in higher complexity. 
The effective dimension captures the impact of the data distribution on the problem's complexity. It often appears in generalization bounds for kernel ridge regression and similar methods, influencing the convergence rates of estimators, 
see e.g. \mycite{caponnetto2007optimal, BlaMuc16}. The rate at which the eigenvalues  $\sigma_j$ decay 
determines the behavior of $\cN(\lambda)$. If the eigenvalues decay as $\sigma_j \lesssim j^{-1/\nu}$, then 
$\mathcal{N}(\la) \lesssim \lambda^{-\nu}$.


\subsection{Gradient Descent}

In the following section, we present the main result on the convergence of the gradient descent algorithm \eqref{Eq:GD-iterate} with a constant stepsize $\eta$. Building on the previously outlined problem setting and the assumptions established in earlier sections, this result provides a rigorous characterization of the algorithm’s behavior under the stated conditions.

We give our error bounds with respect to the scale of norms $||T^u(\cdot)||_{\cH_1}$, $u \in \{0, 1/2\}$, induced by the integral operator 
$T$, arising from the Fr{\'e}chet derivative of the nonlinear operator $A$, allowing to give bounds for both the reconstruction error ($u=0$) and the prediction error 
($u = 1/2$); see Remark~\ref{rem:prediction}, respectively. 

The proof is given in Appendix \ref{app:A-proofs-GD}. 

\vspace{0.2cm}

\begin{theorem}
\label{Thm:GD}
Let Assumptions~\ref{ass:true}, \ref{ass:noise}, \ref{ass:kernel}, \ref{ass:Frechet}, \ref{ass:eta}, \ref{ass:A},   \ref{ass:source} and \ref{ass:eff.bound}  be fulfilled and $\gt$ be the t-th GD iterate \eqref{Eq:GD-iterate}. Assume further that $u \in \{0, 1/2\}$, $\delta \in [\delta_0 , 1]$, 
where $\delta_0 = 6e^{-\frac{d}{C_\star}\eta^{\min\{r, 1/2\}}}$, for some $ C_{\star } >0$, 
depending on $u$, $\eta$, $M$, $\Sigma$, $\kappa_0$, $\lip$, $\alpha$, $C_R$, $\widetilde C_R$, $d$, $D$,  $C_\nu$.   
With probability at least $1-\delta$, we have 
\[\norm{\tp^u (g_t-\fp)}_{\HH_1} \leq C_{\star }\;  
  \cdot (\eta t)^{-(\min\{r,1/2\} + u)} \; \log\left(\frac{6}{\delta}\right)\;, \]
for all $t \in [T_n]$, where $T_n := \floor{\eta^{-1} \;  n^{\frac{1}{\min\{2r,1\}+\nu+1}}}$ and provided $n$ is sufficiently large. 
In particular, $g_t \in \cB_d(\fp)\cap \cD(A)$ for all $t \in [T_n]$. 
\end{theorem}

\vspace{0.2cm}

We briefly comment on the obtained result: The bound shows that algorithm  \eqref{Eq:GD-iterate} is indeed a descent algorithm with high probability, until a threshold $T_n$ 
for the number of iterates is reached, meaning that the error in fact decreases with each GD step.  This threshold $T_n$ depends on the model parameters $r, \nu$, as well as on the stepsize $\eta$.  Note that we require Assumption \ref{ass:A} to hold only locally within a ball of radius $d$ around $\fp$. This local requirement affects the confidence level of our bounds, which increases as the vicinity around the objective $\fp$ expands.

\vspace{0.2cm}

It is now straightforward to derive an upper rate of convergence.  

\vspace{0.2cm}

\begin{corollary}[Rate of Convergence] 
\label{cor_GD-convergence}
Let the assumptions of Theorem \ref{Thm:GD} be satisfied. 
With probability at least $1-\delta$, we have 
\[\norm{\tp^u (g_{T_n}-\fp)}_{\HH_1} \leq C_{\star }\;  
  \cdot \left(\frac{1}{n} \right)^{\frac{\min\{r,1/2\}+u}{2\min\{r,1/2\}+\nu+1}}\; \log\left(\frac{6}{\delta}\right)\;, \]
provided $n$ is sufficiently large.
\end{corollary}

For a more detailed discussion of the result, we refer the reader to Section \ref{sec:discussion}.


\subsection{Stochastic Gradient Descent}

In this section, we present key results on the convergence rates of mini-batch stochastic gradient descent (SGD) under various assumptions on the step size and mini-batch size. We analyze convergence for a constant step size. Furthermore, we explore the impact of different choices for the mini-batch size on the algorithm's performance. Finally, we provide a comprehensive convergence rate result that explicitly accounts for the number of passes over the training data, offering a deeper understanding of the trade-offs involved in practical scenarios.

\vspace{0.2cm}

The proofs are given in Appendix \ref{app:proofs-SGD}.

\vspace{0.2cm}


\begin{theorem}
\label{Thm:SGD.mini}
Let Assumptions~\ref{ass:true}, \ref{ass:bounded}, \ref{ass:kernel}, \ref{ass:Frechet}, \ref{ass:eta}, \ref{ass:A}, \ref{ass:source}, \ref{ass:eff.bound} be fulfilled and $\ft$ be the t-th SGD iterate \eqref{fk.mini}. Assume further that $u \in \{0, 1/2\}$, $\delta \in [\delta_0 , 1]$, 
where $\delta_0 = 6e^{-\frac{d}{C_\star}\eta^{\min\{r, 1/2\}}}$, for some $ C_{\star } >0$
depends on $u$, $\eta$, $\widetilde M$, $\kappa_0$, $\lip$, $\alpha$, $C_R$, $\widetilde C_R$, $d$, $D$,  $C_\nu$.   Assume the mini-batch size obeys 
\begin{equation}\label{b.bound.1}
b\geq \max\paren{\eta (\eta t)^{\min\{2r,1\}+\nu},\eta^{\frac{\min\{2r,1\}+1}{\min\{2r,1\}+\nu+1}} (\eta t)^{\min\{2r,1\}+1}}. 
\end{equation}
With probability at least $1-\delta$, we have the bound 
$$ \EE\sbrac{\norm{\tp^{u}(\ft-\fp)}_{\HH_1}^2} \leq  C_{\star} \cdot (\eta t)^{-(\min\{2r,1\}+2u)}\log^2\paren{\frac{6}{\delta}},$$
for all $t \in [T_n]$, where $T_n := \floor{\eta^{-1} \;  n^{\frac{1}{\min\{2r,1\}+\nu+1}}}$ and provided $n$ is sufficiently large. 
In particular, $f_t \in \cB_d(\fp)\cap \cD(A)$  in expectation, for all $t \in [T_n]$.  The expectation is taken over the $b$-fold uniform distribution on $[n]$ at step $t$.
\end{theorem}

\vspace{0.2cm}

Finally, the above Theorem allows us to derive a rate of convergence under different scenarios. 

\vspace{0.2cm}

\begin{corollary}[Rate of Convergence]
\label{cor:rates-SGD}
Let the assumptions of Theorem \ref{Thm:SGD.mini} be satisfied. Then,  the excess risk of the mini-batch SGD iterate satisfies with probability at least $1-\delta$:
\begin{equation*}
\EE\sbrac{\norm{\tp^{u}(f_{T_n}-\fp)}_{\HH_1}^2} 
\leq  C_{\star} \cdot    n^{-\frac{ \min\{2r,1\}+2u}{ \min\{2r,1\}+\nu+1}}\log^2\paren{\frac{6}{\delta}},    
\end{equation*}
for each of the following choices, given that $n$ is sufficiently large:
\begin{enumerate}[(a)]
    \item $b_n\simeq n^{\frac{\min\{2r,1\}+1}{\min\{2r,1\}+\nu+1}}$, \quad $T_n\simeq n$, \quad $\eta_n\simeq n^{-\frac{\min\{2r,1\}+\nu}{\min\{2r,1\}+\nu+1}}$, \qquad \text{($n^{\frac{\min\{2r,1\}+1}{\min\{2r,1\}+\nu+1}}$ passes over data)}
    \item $b_n\simeq n^{\frac{\min\{2r,1\}+1}{\min\{2r,1\}+\nu+1}}$, \quad $T_n\simeq n^{\frac{1}{\min\{2r,1\}+\nu+1}}$, \quad $\eta_n\simeq 1$, \qquad \text{($n^{\frac{1-\nu}{\min\{2r,1\}+\nu+1}}$  passes over data)}
    \item $b_n\simeq n$, \quad $T_n\simeq n^{\frac{1}{\min\{2r,1\}+\nu+1}}$, \quad $\eta_n\simeq 1$, \qquad \text{($n^{\frac{1}{\min\{2r,1\}+\nu+1}}$ passes over data)}
    \item $b_n\simeq 1$, \quad $T_n\simeq n^{\frac{\min\{2r,1\}+\nu+2}{\min\{2r,1\}+\nu+1}}$, \quad $\eta_n\simeq n^{-1}$, \qquad \text{($n^{\frac{1}{\min\{2r,1\}+\nu+1}}$ passes over data)}.
\end{enumerate}
\end{corollary}

\vspace{0.2cm}

These results highlight the trade-offs between the batch size, learning rate, and computational budget. All these different choices lead to the same rate of convergence.  Increasing the batch size reduces the variance of gradient estimates, potentially leading to more stable convergence. However, larger batches require more computational resources and may lead to fewer parameter updates per epoch, slowing down the progress in terms of time, in particular when combined with a small step size (case $(a)$). Using smaller batches increases the variance in gradient estimates. This approach is computationally cheaper per update but requires more updates overall (case $(d)$). 

We finally point out that for batch GD, i.e. $b_n=n$ (case $(c)$), we recover a randomized version of 
Corollary \ref{cor_GD-convergence} with the same stopping time $T_n$ and constant $\eta$. 
\vspace{0.2cm}

For a more detailed discussion of the result, we refer the reader to Section \ref{sec:discussion}. 

\vspace{0.3cm}

\begin{remark}\label{rem:prediction}
The error bound in the interpolation norm for $u=\frac{1}{2}$ provides the convergence rates in the prediction norm. This can be observed from the following calculations. 

Using the Taylor expansion, we get
\begin{align*}
\norm{\ip \sbrac{A(f)-A(\fp)}}_{\LL} \leq \norm{\ip A'(\fp)(f-\fp)}_{\LL}+\norm{\ip r(f)}_{\LL}
\end{align*}
From Lemma~\ref{lem:taylor}, we obtain
\begin{align*}
\norm{\ip \sbrac{A(f)-A(\fp)}}_{\LL} \leq & \norm{\ip A'(\fp)(f-\fp)}_{\LL}+\frac{C_R}{2}\cdot \norm{f-\fp}_{\HH_1} \cdot \norm{\ip A'(\fp)(f-\fp)}_{\LL}\\
= & \paren{1+\frac{C_R}{2}\cdot \norm{f-\fp}_{\HH_1}}\cdot \norm{\tp^{\frac{1}{2}}(f-\fp)}_{\HH_1}.
\end{align*}
\end{remark}


\section{Discussion of Results}
\label{sec:discussion}

In this section, we provide a more detailed comparison of our results with existing ones in the context of inverse learning. 

\vspace{0.2cm}

{\bf GD for linear inverse problems.} 
Under model \eqref{model}, when $A: \cH_1 \to \cH_2$ is linear, a convergence analysis of GD with a constant step size has been established in \mycite{BlaMuc16} and \mycite{Rastogi17} under our given assumptions (noise model, H\"older source condition, and polynomial eigenvalue decay). Interestingly, our convergence rate in Corollary \ref{cor_GD-convergence} with the chosen stopping time $T_n$ matches those given in these references, which are known to be minimax optimal. In Theorem \ref{Thm:GD}, we extend these results by showing that GD is a descent algorithm even in the random design setting, with high probability, meaning that the error decreases with each consecutive GD step until a threshold $T_n$ for the number of steps is reached. Notably, this threshold is identified in the linear case as the optimal stopping time, thereby ensuring optimality. 



\vspace{0.3cm}

{\bf SGD for linear inverse problems.} In \mycite{lu2022stochastic}, the authors analyze a minibatch version of SGD for solving linear inverse problems in a fixed design setting (i.e., where the samples are deterministic). They propose a projected version of SGD to discretize an original infinite-dimensional linear operator $A$, where the use of randomly chosen minibatches introduces noise fluctuations. Their results derive error bounds that account for factors such as discretization levels, the decay rate of the step size, and the smoothness of the problem, characterized by general source conditions. The analysis accommodates a wide range of step size schedules, requiring them to be non-increasing with $\eta_t \log(t) \to 0$ and $\eta_t t \to \infty$ as $t \to \infty$, including power-type decays $\eta_t \sim t^{-\alpha}$ for $\alpha > 0$. Since the fixed design framework and the obtained results differ significantly from our random design setting, we only attempt to qualitatively compare their findings with ours.

A key insight in \mycite{lu2022stochastic} is that the convergence rate of mini-batch SGD is independent of the batch size. However, the a priori chosen stopping time $T^*$ does depend on the batch size. Specifically, larger batch sizes correspond to smaller values of $T^*$ for the same stepsize schedule. While our rate of convergence is also independent of the chosen batch size, the observation in \mycite{lu2022stochastic} is fundamentally different from our findings in Corollary \ref{cor:rates-SGD}. Cases $(b)$ and $(c)$ reveal that for a constant stepsize with an optimal stopping time $T_n$, increasing the batch size beyond the threshold given in $(b)$ - up to a batch size of $b_n = n$ (full-batch GD) - offers no additional benefit. In fact, it even increases the computational complexity as the number of passes over the data grows. The same observation is made in \mycite{mucke19} (Corollary 1) for minibatch tail-averaging SGD for linear prediction problems (forward problem) in random design. However, while in 
\mycite{mucke19}, one pass over the data is possible (case $(c)$ in Corollary 1), we require at least $\cO(n^{\frac{1}{\min\{2r,1\}+\nu+1}})$ passes to obtain optimality.



\vspace{0.3cm}

{\bf Nonlinear inverse problems.} The convergence analysis of algorithms for nonlinear inverse problems in random design is, to the best of our knowledge, restricted to Tikhonov regularization (see the stream of works \mycite{Rastogi20}, \mycite{Rastogi20a}, \mycite{Rastogi24}). The estimator is defined as the minimizer of a Tikhonov functional, which consists of the sum of a data misfit term and a quadratic penalty term. In these works, a theoretical analysis of the minimizer of the Tikhonov regularization scheme using the concept of reproducing kernel Hilbert spaces is developed. In \mycite{Rastogi20}, upper rates are presented for smoothness $r \in [1/2, 1]$; however, there is no restriction on the range of $r$ for lower rates. In contrast, we derive matching upper bounds for a range of smoothness with $r \in (0, 1/2]$, showing that GD and SGD seem to suffer from saturation in the nonlinear setting. This stands in contrast to linear inverse problems, where GD and SGD with tail-averaging do not experience this phenomenon \mycite{BlaMuc16}, \mycite{mucke19}.

Despite its good theoretical performance, Tikhonov regularization in the nonlinear setting has the drawback of being impossible to implement, as it lacks a closed-form solution, unlike in the linear case. As a result, an approximate estimator for $\fp$ must be derived, which can also be implemented numerically. We extend the established theory of Tikhonov regularization and propose two distinct iterative algorithms to address this limitation.

Iterative regularization approaches for nonlinear inverse problems in Hilbert spaces in the fixed design setting are better understood. Despite the differing settings, we compare our results for GD with a constant step size to the convergence results in \mycite{Hanke95, kaltenbacher08} for constant step size GD. Regarding smoothness, those earlier works also require $r \in (0, 1/2]$, which aligns with our findings. The reconstruction error is then bounded by 
\begin{equation}
\label{eq:hanke}
 || g_t - \fp ||_{\cH_1} \leq  C_\star\; t^{-r} \;, 
\end{equation} 
for some $C_\star < \infty$ and for all $t \leq T^\star$, where $T^\star$ is an early stopping time derived from the discrepancy principle. Notably, this bound coincides with our bound in Theorem \ref{Thm:GD} and holds under the same Assumption \ref{ass:A} on the nonlinear operator. 

 A convergence analysis of SGD in fixed design with batch size one and decaying stepsize $\eta_t = \eta_0t^{-\alpha}$ with $\alpha \in (0,1)$ and $\eta_0$ sufficiently small is conducted in 
 \mycite{jin2020convergence}. Under Assumption \ref{ass:A}, the reconstruction error is bounded by 
 \begin{equation}
\label{eq:jin}
 \mbe[ || f_t - \fp ||_{\cH_1} ] \leq  C_\star\; t^{-min\{ r(1-\alpha), \alpha - \epsilon\}} \;, 
\end{equation} 
for some $\epsilon \in (0, \alpha /2)$ and  $C_\star < \infty$. Again, this bound holds for all $t$ up to a threshold $T^*$ (depending on the noise level) and is even slightly better than our result in Corollary \ref{cor:rates-SGD}, case $(d)$.

To sum up, both results \eqref{eq:hanke} and \eqref{eq:jin} also show the descent property of GD and SGD (in expectation), similarly to our findings in Theorem \ref{Thm:GD} and Theorem \ref{Thm:SGD.mini}, respectively.

{\bf Acknowledgements:} This research has been partially supported by the Research Council of Finland (decision number 353094).


\bibliography{bib_SGD, references}



\appendix


\section{Preliminaries}
\label{app:prelim}

\subsection{Notation}

Here, we summarize the notations used in our analysis. Some of the notations have already been introduced before. We can classify our operators into the population version and the empirical version. First, the population version operators depend on the marginal probability measure $\nu$. On the other hand, the empirical version operators depend on the inputs $(x_1,\ldots, x_n)$. 

\vspace{0.3cm}

{\bf Population version operators:} Here, we remind that the nonlinear operator $A:\cD(A)\subset \cH_1 \to \cH_2$ and $\ip: \cH_2 \hookrightarrow L^2(\cX, \nu)\;$ is the inclusion operator. We define the following population version operator depending on the marginal measure $\nu$ through $\ip$:
\begin{align*}
  F &:=\ip\circ A: \cD(A) \to L^2(\cX, \nu)\;, \\
 \bp&:=\ip\circ A'(\fp): \cD(A) \to \cH_2 \hookrightarrow L^2(\cX, \nu)\;,  \\
 \tp&:= \bp^*\bp: \cD(A) \to \cH_1\; ,   \\
 \lp&:=\ip^*\ip :  \cH_2 \to  \cH_2. 
\end{align*} 

{\bf Empirical version operators:} For $f \in \cD(A)$ and the sampling operator $\sx:f \mapsto (f(x_1),\ldots,f(x_n))$ we define the empirical version operators:
\begin{align*}
    \hat{F} &:=\sx\circ A: \cD(A) \to \cY^n, \\
    \bx &:=\sx\circ A'(\fp): \cD(A) \to \cY^n, \\
     \hat{B}_{f}&:=\sx\circ A'(f): \cD(A) \to \cY^n\;,\\
    \tx &:= \bx^*\bx: \cD(A) \to \cH_1,\\ 
    \lx &:=\sx^*\sx :  \cH_2 \to \cH_2.
\end{align*}

For $f \in \cD(A)$ and the pointwise evaluation map $S_{x_i}:f \mapsto f(x_i)$, $1\leq i\leq n$ we define the empirical version operators. The following operators will be used particularly in the analysis of the SGD results.
\begin{align*}
F_i:=S_{x_i}\circ A: \cD(A) \to \cY, \\
B_i:=S_{x_i}\circ A'(\fp): \cD(A) \to \cY, \\
B_{f,i}:=S_{x_i}\circ A'(f): \cD(A) \to \cY.
\end{align*}

Further, several probabilistic quantities will be used to express the error bounds. 
\begin{definition}
\label{def:norms}
For  $s>0$, $\VE=(\sx \circ A)(\fp)-\yy$,  we denote
\begin{align*}\label{def:norms}
 \Xi^{s}(\la) &:= \norm{(\lx+\la I)^{-s}(\lp+\la  I)^s},\\
\Psi_{\xx}(\la)&: =\norm{(\lp+\la I)^{-1/2}(\lp-\lx)},\\ 
\Upsilon_{\xx}(\la)&: =\tr\sbrac{(\lp+\la I)^{-1}(\lp-\lx)},\\
\Theta_{\zz}(\la)&:= \norm{(\lp+\la I)^{-1/2}\sx^*\VE}_{\HH_2}. 
\end{align*}
\end{definition}


\subsection{Preliminary Results}

The following lemma provides the bound for the remainder term using the Range Invariance condition.

\begin{lemma}[Taylor Remainder]
\label{lem:taylor}
Let Assumptions \ref{ass:Frechet} and \ref{ass:A} be satisfied. 
Then for any $f\in \cB_d(\fp)\cap \cD(A)$, we have the Taylor expansion 
\begin{equation*}
\label{Taylor_exp_Rla}
A(f)=A(\fp)+A'(\fp)(f-\fp)+r(f)\;,
\end{equation*}
where $r(f)$ denotes the remainder term.  Moreover, the remainder admits the following bound: 
\[ \|\ip r(f)\|_{\LL} \leq \frac{C_R}{2} \cdot \norm{f-\fp}_{\HH_1} \cdot \norm{\bp\paren{f-\fp}}_{\LL}. \]
\end{lemma}


\begin{proof}[Proof of Lemma \ref{lem:taylor}]
Under the above assumptions, we find 
\begin{align*}
\label{Rla}
\|\ip r(f)\|_{\LL} =& \norm{\ip\brac{A(f)-A(\fp)-A'(\fp)(f-\fp)}}_{\LL} \\ \nonumber   
= &\norm{\ip \int_{0}^1  \brac{A'\paren{\fp+t(f-\fp)}-A'(\fp)}\paren{f-\fp}dt}_{\LL}\\ \nonumber
= &\norm{\int_{0}^1  \brac{R_{\fp +t(f-\fp )}\ip A'(\fp )-\ip A'(\fp )}\paren{f-\fp }dt}_{\LL}\\ \nonumber
= &\norm{\int_{0}^1  \brac{R_{\fp +t(f-\fp )}-I}\ip A'(\fp )\paren{f-\fp }dt}_{\LL}\\ \nonumber
\leq& C_R\norm{f-\fp }_{\HH_1}\norm{\bp\paren{f-\fp }}_{\LL}\int_{0}^1 tdt \\  \nonumber
\leq & \frac{C_R}{2}\norm{f-\fp }_{\HH_1}\norm{\bp\paren{f-\fp }}_{\LL}.
\end{align*}
\end{proof}



\section{Proofs Gradient Descent}
\label{app:A-proofs-GD}


For proving our error bounds, we need some intermediate calculations, provided in the next Lemma.

\vspace{0.2cm}

\begin{lemma}
\label{Lemma:sum.phi}
Let $a \in [0,1)$ , $b \in [0,1)$, $c \in [0, \infty)$, $d \in (0, \infty)$. 
For $u \in [0,1]$ and $j \in \mbn$, define 
\begin{equation}\label{phi}
\phi_j^u :=\paren{\frac{u}{u+j}}^u \quad \text{with} \quad  \phi_j^0: =1.    
\end{equation}
Assume $k \geq 2$. The following inequalities hold true: 
\begin{enumerate}
\item 
\begin{equation}
\label{lem.phi.1}
\sum\limits_{j=1}^k \phi_{k-j}^{\frac{1}{2}} j^{-b}   \leq \beta\paren{{\frac{1}{2}},-b+1}(k+1)^{{\frac{1}{2}}-b},
\end{equation}

\item 
\begin{equation}
\label{lem.phi.2}
\sum\limits_{j=1}^{k} \paren{\phi_{k-j}^{\frac{1}{2}}}^2  j^{-\paren{b+d}}
\leq  2^{b}\brac{\paren{1-b}^{-1}+2^{d+1}(ed)^{-1}}(k+1)^{-b},
\end{equation}

\item 
\begin{equation}
\label{lem.phi.3}
\sum\limits_{j=1}^{k} \phi_{k-j}^1  j^{-\paren{b+d}}
\leq  2^{b}\brac{\paren{1-b}^{-1}+2^{d+1}(ed)^{-1}}(k+1)^{-b},
\end{equation}




\item 
\begin{equation}\label{lem:tx.ak}
   \sum\limits_{j=1}^t \paren{\phi_{t-j}^{1}}^v \leq \begin{cases}
(1-v)^{-1}\paren{t+1}^{1-v}, & \quad v\in[0,1),\\
        2(ed)^{-1}(t+1)^d, & \quad v= 1,\\
        v(v-1)^{-1}, & \quad v> 1.
\end{cases}
\end{equation}
\end{enumerate}  
\end{lemma}

\begin{proof}[Proof of Lemma \ref{Lemma:sum.phi}]
\begin{enumerate}
\item 
We start showing inequality \eqref{lem.phi.1}. We have 
\begin{equation*}
    \sum\limits_{j=1}^k \phi_{k-j}^{\frac{1}{2}} j^{-b} = \sum\limits_{j=1}^k  \paren{\frac{{\frac{1}{2}}}{{\frac{1}{2}}+k-j}}^{\frac{1}{2}} j^{-b}.
\end{equation*} 
The inequality $\paren{\frac{{\frac{1}{2}}}{{\frac{1}{2}}+k-j}}^{\frac{1}{2}}  \leq \paren{\frac{1}{1+k-j}}^{\frac{1}{2}}$ yields
\begin{equation}
\label{sum.phi.bnd}
\sum\limits_{j=1}^k \phi_{k-j}^{\frac{1}{2}} j^{-b}   \leq \sum\limits_{j=1}^k  \paren{\frac{1}{1+k-j}}^{\frac{1}{2}} j^{-b} \leq \int_0^{k+1} \paren{\frac{1}{1+k-x}}^{\frac{1}{2}} x^{-b} dx.
 \end{equation}
By substituting $x=(k+1)y$ for $0\leq b< 1$ we obtain 
\begin{equation}
\label{sum.phi.bd}
\sum\limits_{j=1}^k \phi_{k-j}^{\frac{1}{2}} j^{-b}   \leq (k+1)^{{\frac{1}{2}}-b} \int_0^{1} (1-y)^{-{\frac{1}{2}}} y^{-b} dy = (k+1)^{{\frac{1}{2}}-b} \beta\paren{{\frac{1}{2}},-b+1}.
\end{equation}
This provides the proof of the first inequality \eqref{lem.phi.1}. 

\item 
Similarly, the second inequality \eqref{lem.phi.2} can be derived as follows: 
\begin{equation}\label{phi.frac}
    \sum\limits_{j=1}^k \phi_{k-j}^{\frac{1}{2}} j^{-(b+d)} 
= \sum\limits_{j=1}^k  \paren{\frac{{\frac{1}{2}}}{{\frac{1}{2}}+k-j}}^{\frac{1}{2}}j^{-(b+d)}
\leq \sum\limits_{j=1}^k  \paren{\frac{1}{1+k-j}}^{\frac{1}{2}} j^{-(b+d)}.
\end{equation}

We write
\begin{align*}
\sum\limits_{j=1}^{k} \paren{\frac{1}{1+k-j}}^{\frac{1}{2}}  j^{-\paren{b+d}} 
= & \sum\limits_{j=1}^{\sbrac{\frac{k+1}{2}}} \paren{\frac{1}{1+k-j}}^{\frac{1}{2}}  j^{-(b+d)} + 
       \sum\limits_{j=\sbrac{\frac{k+1}{2}}+1}^{k} \paren{\frac{1}{1+k-j}}^{\frac{1}{2}}  j^{-\paren{b+d}}    \\
\leq & \paren{\frac{2}{k+1}}^{\frac{1}{2}}\sum\limits_{j=1}^{\sbrac{\frac{k+1}{2}}}   j^{-(b+d)} + 
         \paren{\frac{2}{k+1}}^{b+d}\sum\limits_{j={\sbrac{\frac{k+1}{2}}+1}}^k  \paren{\frac{1}{1+k-j}}^{\frac{1}{2}}.
\end{align*}

A simple computation gives
\begin{equation}\label{sum.phi.1.1}
    \sum\limits_{j={\sbrac{\frac{k+1}{2}}+1}}^k  \paren{\frac{1}{1+k-j}}^{\frac{1}{2}} 
= \sum\limits_{j={1}}^{\sbrac{\frac{k+1}{2}}-1}  \paren{\frac{1}{1+k-(\sbrac{\frac{k+1}{2}}+j)}}^{\frac{1}{2}}
\leq 2\sqrt{k+1},
\end{equation}
where we apply the first part of Lemma 14 in \cite{mucke19}. More precisely, setting $A= 1+\sbrac{\frac{k+1}{2}}$, we have 
\[ \sum\limits_{j={1}}^{\sbrac{\frac{k+1}{2}}-1} \paren{\frac{1}{1+k-(\sbrac{\frac{k+1}{2}}+j)}}^{\frac{1}{2}} \leq \int_{1}^{\sbrac{\frac{k+1}{2}}} 
\paren{\frac{1}{A-x}}^{\frac{1}{2}}\leq 2 \sqrt{A-1} .\]
Moreover, by Lemma D.7 in \cite{Nguyen2024} we have 
\begin{equation}\label{sum.phi.1.2}
    \sum\limits_{j=1}^{\sbrac{\frac{k+1}{2}}}   j^{-(b+d)} 
\leq 
\begin{cases}
(1-(b+d))^{-1}\paren{\frac{k+1}{2}}^{1-(b+d)}, & \quad b+d\in[0,1),\\
2\ln(k+1), & \quad b+d= 1,\\
(b+d)(b+d-1)^{-1}, & \quad b+d> 1.
\end{cases}
\end{equation}

Using the estimates~\eqref{sum.phi.1.1} and \eqref{sum.phi.1.2}, we obtain: 
\begin{equation*}
\sum\limits_{j=1}^{k} \paren{\frac{1}{1+k-j}}^{\frac{1}{2}}  j^{-\paren{b+d}} 
\leq  2^b(1-b)^{-1}\paren{k+1}^{-b}+2^{1+b+d} (k+1)^{-\paren{b+d+1/2}}.
\end{equation*}
Similarly, for $b<1$ we get: \begin{equation}\label{sum.frac}
\sum\limits_{j=1}^{k} \paren{\frac{1}{1+k-j}}^{\frac{1}{2}}  j^{-b} \leq  2^b(1-b)^{-1}\paren{k+1}^{-b}+2^{1+b} (k+1)^{-b}\ln (k+1).
\end{equation}

By the inequality $s^{-d}\ln s\leq (ed)^{-1}$, we conclude
\begin{align}\label{sum.phi.bd.4}
\sum\limits_{j=1}^{k} \paren{\frac{1}{1+k-j}}  j^{-\paren{b+d}}
\leq & 2^b(1-b)^{-1}\paren{k+1}^{-b}+2^{1+b+d} (k+1)^{-b}(ed)^{-1}\\  \nonumber
\leq & 2^{b}\brac{\paren{1-b}^{-1}+2^{d+1}(ed)^{-1}}(k+1)^{-b}.
\end{align}

Combining the estimates \eqref{phi.frac} and \eqref{sum.phi.bd.4} implies the assertion:
\begin{equation*}
     \sum\limits_{j=1}^k \paren{\phi_{k-j}^{\frac{1}{2}}}^2 j^{-(b+d)} \leq   2^{b}\brac{\paren{1-b}^{-1}+2^{d+1}(ed)^{-1}}(k+1)^{-b}.
\end{equation*}

\item
We continue to show the third inequality \eqref{lem.phi.3}. We have
\begin{equation*}
\sum\limits_{j=1}^{k} \phi_{k-j}^1  j^{-\paren{b+d}} =      \sum\limits_{j=1}^{k} \paren{\frac{1}{1+k-j}}  j^{-\paren{b+d}}.
\end{equation*}

Using the inequality \eqref{sum.phi.bd.4} we obtain
\begin{equation}
     \sum\limits_{j=1}^k \phi_{k-j}^1 j^{-(b+d)} \leq   2^{b}\brac{\paren{1-b}^{-1}+2^{d+1}(ed)^{-1}}(k+1)^{-b}.
\end{equation}
This proves the estimate \eqref{lem.phi.3}.

\item
Now, we prove the inequality \eqref{lem:tx.ak}. For $0<v< 1$, we have
\begin{align*}
&\sum\limits_{j=1}^t \paren{\phi_{t-j}^{1}}^v = \sum\limits_{j=1}^t\paren{\frac{1}{1+t-j}}^{v}.
\end{align*}
By substituting $i=1+t-j$, we get
\begin{align*}
&\sum\limits_{j=1}^t \paren{\phi_{t-j}^{1}}^v 
=\sum\limits_{i=1}^t \frac{1}{i^{v}}.
\end{align*}
From \eqref{sum.phi.1.2} we obtain 
\begin{equation*}
   \sum\limits_{j=1}^t \paren{\phi_{t-j}^{1}}^v \leq \begin{cases}
(1-v)^{-1}\paren{t+1}^{1-v}, & \quad v\in[0,1),\\
        2\ln(t+1), & \quad v= 1,\\
        v(v-1)^{-1}, & \quad v> 1.
\end{cases}
\end{equation*}

By the inequality $s^{-d}\ln s\leq (ed)^{-1}$ for some $d>0$, we conclude
\begin{equation*}
   \sum\limits_{j=1}^t \paren{\phi_{t-j}^{1}}^v \leq \begin{cases}
(1-v)^{-1}\paren{t+1}^{1-v}, & \quad v\in[0,1),\\
        2(ed)^{-1}(t+1)^d, & \quad v= 1,\\
        v(v-1)^{-1}, & \quad v> 1.
\end{cases}
\end{equation*}
\end{enumerate}  
\end{proof}


\subsection{Proof of Theorem \ref{Thm:GD}}

We continue with a Lemma, providing a formula for the $t-th$ GD step.

\begin{lemma}
\label{lem:GD-prelim}
Let $j \in \mbn$ and $\lam_j >0$. For any $t \geq 1$ we have 
\begin{equation*}
\gto-\fp=(I-\eta\tp)^t(g_1-\fp)-\eta \sum\limits_{j=1}^t  (I-\eta\tp)^{t-j} \bp^* v_j
      - \eta\sum\limits_{j=1}^t  (I-\eta\tp)^{t-j}  w_j\;,
\end{equation*}

where 
\begin{align*}
v_j &= R_{g_j}^*\ip\brac{r(g_j)+(\lp+\la_j I)^{-1} \brac{\paren{\lx-\lp}\paren{A(g_j)-A(\fp)}-\sx^*\VE}}\;,\\
w_j &=  \lam_jA'(g_j)^*(\lp+\la_j I)^{-1}\brac{ \paren{\lx-\lp}\paren{A(g_j)-A(\fp)} -\sx^*\VE} \;.
\end{align*}
\end{lemma}

\vspace{0.3cm}

\begin{proof}[Proof of Lemma \ref{lem:GD-prelim}]
For the error term $e_t=g_t-\fp$, we have from the definition of the Gradient Descent algorithm ~\eqref{Eq:GD-iterate} 
and by denoting $\VE=\yy-\hat F(\fp)$
\begin{align*}
    e_{t+1}&= \ek -\eta \hat B^*_{g_t} \paren{ \hat F(\gt)-\yy} \\
&=  \ek -\eta A'(\gt)^*\lp\brac{ A(\gt)-A(\fp)} -\eta A'(\gt)^*\brac{ \paren{\lx-\lp}\paren{A(\gt)-A(\fp)} -\sx^*\VE}\,. 
\end{align*}
Let $\lam_t >0$. Denoting 
\begin{align*}
r(\gt)=&A(\gt)-A(\fp)-A'(\fp)(\gt-\fp),\\ 
v_t=&R_{\gt}^*\ip\brac{r(\gt)+(\lp+\la_t I)^{-1} \brac{\paren{\lx-\lp}\paren{A(\gt)-A(\fp)}-\sx^*\VE}}\\
&+(R_{\gt}^*-I)\bp(\gt-\fp)\\ \intertext{and} 
w_t=& \lam_t A'(\gt)^*(\lp+\la_t I)^{-1}\brac{ \paren{\lx-\lp}\paren{A(\gt)-A(\fp)} -\sx^*\VE} \;,  
\end{align*}
we obtain by using Assumption~\ref{ass:A} (iii) the following recursion:
\begin{align*}
e_{t+1} &=   (I-\eta \tp)\ek -\eta \bp^* v_t -  \eta w_t \;.
\end{align*}
Inductively, we obtain
\begin{equation*}
e_{t+1}=(I-\eta\tp)^te_1-\eta \sum\limits_{j=1}^t  (I-\eta\tp)^{t-j} \bp^* v_j-  \eta\sum\limits_{j=1}^t  (I-\eta\tp)^{t-j}  w_j\;, 
\end{equation*}
proving our result. 
\end{proof}


\vspace{0.3cm}

\begin{proof}[Proof of Theorem \ref{Thm:GD}] 
The proof is divided into two main steps: First, we derive an upper bound for the reconstruction and prediction error, respectively. 
In the second step, we prove the main result by induction. 

\vspace{0.3cm}

{\bf Step 1: Deriving upper bound.} We set $\et=\gt-\fp$. From the source condition Assumption \ref{ass:source}, 
we have $e_1=\fp - g_1 = \tp^r g$. Applying Lemma \ref{lem:GD-prelim} we get 
\begin{align}
\norm{\tp^u \eto}_{\HH_1} 
&\leq  \norm{\tp^u(I-\eta\tp)^t\tp^r g}_{\HH_1} 
+ \eta \sum\limits_{j=1}^t  \norm{\tp^u(I-\eta\tp)^{t-j}\bp^*} \cdot  \norm{v_j}_{\LL} \nonumber \\  
&+  \eta \sum\limits_{j=1}^t  \norm{\tp^u(I-\eta\tp)^{t-j}} \cdot  \norm{w_j}_{\HH_1},
\end{align}
where $v_j$ and $w_j$ are defined in Lemma \ref{lem:GD-prelim}. 

By Lemma D.6 in \cite{Nguyen2024}, 
for $u>0$, we obtain for $\eta\leq 1/\kao^2$ and the notation \eqref{phi}:
\begin{equation}
\label{phi.bound}
    \norm{(\eta\tp)^u(I-\eta \tp)^j} \leq \sup\limits_{x\in[0,1]}x^u(1-x)^j\leq \paren{\frac{u}{u+j}}^u = \phi_j^u.
\end{equation}

This gives 
\begin{align}
\label{eq:main-GD}
\norm{\tp^u \eto}_{\HH_1} 
&\leq \underbrace{D \eta^{-(u+r)}\phi_t^{u+r}}_{T_1} + 
      \underbrace{\eta^{\frac{1}{2}-u} \sum\limits_{j=1}^t  \phi_{t-j}^{u+\frac{1}{2}}   \norm{v_j}_{\LL}}_{T_2}  + 
      \underbrace{ \eta^{1-u} \sum\limits_{j=1}^t  \phi_{t-j}^{u}  \norm{w_j}_{\HH_1} }_{T_3}\;.
\end{align}

In what follows, we bound the 
individual parts $T_2$ and $T_3$ separately to simplify 
the proof. 

\vspace{0.3cm}



{\bf Bounding $T_2$.}  By the definition of $v_j$ we obtain
\begin{align*}
\norm{v_j}_{\LL} &\leq \norm{R_{g_j}^*\ip r(g_j)} + \norm{R_{g_j}^*\ip (\lp+\la_t I)^{-1} (\lx-\lp)(A(g_j) - A(\fp))} \nonumber \\
&+ \norm{R_{g_j}^*\ip (\lp+\la_t I)^{-1}\sx^*\VE } + \norm{(R_{g_j}^*-I)\bp(g_j-\fp)} \;.
\end{align*}
By Assumption \ref{ass:A}, applying Lemma \ref{lem:taylor}, by \eqref{eq:lipschitz3} and invoking the notation from Definition \ref{def:norms}, we find 
\begin{align*}
\norm{v_j}_{\LL} &\leq \frac{\tilde C_R \cdot C_R}{2} \cdot \norm{g_j-\fp}_{\HH_1} \cdot \norm{\bp\paren{g_j-\fp}}_{L^2}\nonumber \\
 &+  \tilde C_R\; \Psi_{\xx}(\la_t) \; ||A(g_j) - A(\fp)|| + \tilde C_R \; \Theta_{\zz}(\la_t) + 
C_R \; ||g_j-\fp||_{\cH_1} \cdot \norm{\bp\paren{g_j-\fp}}_{L^2} \nonumber \\ 
&\leq \frac{\tilde C_R \cdot C_R}{2} \cdot \norm{g_j-\fp}_{\HH_1} \cdot \norm{\bp\paren{g_j-\fp}}_{L^2}\nonumber \\
 &+  \frac{\tilde C_R \cdot \lip }{1-\alpha}\cdot \Psi_{\xx}(\la_t) \cdot ||g_j-\fp||_{\cH_1} + \tilde C_R \; \Theta_{\zz}(\la_t) + 
C_R \; ||g_j-\fp||_{\cH_1} \cdot \norm{\bp\paren{g_j-\fp}}_{L^2} \;.
\end{align*}
Hence, setting 
\[ a_j := ||g_j-\fp||_{\cH_1} \;, \quad b_j:=\norm{\bp\paren{g_j-\fp}}_{L^2} \]
yields 
\begin{align}
\label{eq:T2-SGD}
T_2 &\leq C^{(1)}_{\alpha, \lip} \; \eta^{\frac{1}{2}-u} \sum\limits_{j=1}^t  \phi_{t-j}^{u+\frac{1}{2}} \left( a_j\cdot b_j 
+   \Psi_{\xx}(\la_t) \cdot a_j + \Theta_{\zz}(\la_t)  \right) \;, 
\end{align}
where 
\[ C^{(1)}_{\alpha, \lip} =  \max\left\{  \frac{\tilde C_R \cdot C_R}{2} ,  \frac{\tilde C_R \cdot \lip }{1-\alpha},  \tilde C_R , C_R \right\}  \;. \] 

\vspace{0.2cm}

{\bf Bounding $T_3$.}  By the definition of $w_j$ we obtain by Assumption \ref{ass:Frechet}, by \eqref{eq:lipschitz3}
and Definition \ref{def:norms} 
\begin{align*}
\norm{w_j}_{\cH_1} 
&=  \la_t    \norm{ A'(g_j)^*(\lp+\la_t I)^{-1}\brac{ \paren{\lx-\lp}\paren{A(g_j)-A(\fp)} -\sx^*\VE } }_{\cH_1} \nonumber\\
&\leq \la_t \norm{ A'(g_j)^*(\lp+\la_t I)^{-1} \paren{\lx-\lp} \cdot \paren{A(g_j)-A(\fp)}}_{\cH_1} + 
     \la_t  \norm{A'(g_j)^*(\lp+\la_t I)^{-1} \sx^*\VE  }_{\cH_1} \nonumber\\
&\leq \la_t \norm{A'(g_j)^*(\lp+\la_t I)^{-\frac{1}{2}}} \cdot \norm{(\lp+\la_t I)^{-\frac{1}{2}} \paren{\lx-\lp}} \cdot ||A(g_j)-A(\fp)||_{\cH_2} \nonumber \\ 
& + \;\; \;\la_t \norm{A'(g_j)^*(\lp+\la_t I)^{-\frac{1}{2}}} \cdot \norm{(\lp+\la_t I)^{-\frac{1}{2}} \sx^*\VE  }_{\cH_1} \nonumber\\
&\leq  \lip \cdot \sqrt{\la_t} \cdot \Psi_{\xx}(\la_t) \cdot ||A(g_j)-A(\fp)||_{\cH_2} +  \lip \cdot \sqrt{\la_t} \cdot \Theta_{\zz}(\la_t) \nonumber \\
&\leq  \frac{\lip^2\sqrt{\la_t} }{1-\alpha}  \cdot \Psi_{\xx}(\la_t) \cdot ||g_j-\fp||_{\cH_1} +  \lip \cdot \sqrt{\la_t}\cdot \Theta_{\zz}(\la_t)  \;.
\end{align*}

Hence, 
\begin{align}
\label{eq:T3-SGD}
T_3 &\leq  C^{(2)}_{\alpha, \lip} \; \eta^{1-u}  \; 
 \sum\limits_{j=1}^t  \phi_{t-j}^{u} \sqrt{\la_t} \cdot \left(  \Psi_{\xx}(\la_t) \cdot a_j 
+  \Theta_{\zz}(\la_t) \right)  \;,
\end{align}
where we set 
\[ C^{(2)}_{\alpha, \lip} = \max\left\{ \frac{\lip^2 }{(1-\alpha)} , \lip \right\}  \;. \]

\vspace{0.2cm}

Combining \eqref{eq:main-GD}, \eqref{eq:T2-SGD}, and \eqref{eq:T3-SGD}  yields 
\begin{align}
\label{eq:main-SGD2}
\norm{\tp^u \eto}_{\HH_1} 
&\leq D \eta^{-(u+r)}\phi_t^{u+r} + 
C^{(1)}_{\alpha, \lip} \; \eta^{\frac{1}{2}-u} \sum\limits_{j=1}^t  \phi_{t-j}^{u+\frac{1}{2}} \left( a_j\cdot b_j 
+   \Psi_{\xx}(\la_t) \cdot a_j + \Theta_{\zz}(\la_t) \right)\nonumber  \\
&+ \;\;\;  C^{(2)}_{\alpha, \lip} \; \eta^{1-u}  \; \sum\limits_{j=1}^t  \phi_{t-j}^{u} \sqrt{\la_t} \cdot \left(  \Psi_{\xx}(\la_t) \cdot a_j 
+  \Theta_{\zz}(\la_t) \right) \nonumber \\
&=  D \eta^{-(u+r)}\phi_t^{u+r} +  \sum_{i=1}^3 S_i(u,t)\;,
\end{align}
where we set
\begin{align*}
S_1(u,t)&:= C^{(1)}_{\alpha, \lip} \; \eta^{\frac{1}{2}-u}  \sum\limits_{j=1}^t  \phi_{t-j}^{u+\frac{1}{2}} \; a_j\cdot b_j  \;,\\
S_2(u,t)&:=  C^{(1)}_{\alpha, \lip} \; \eta^{\frac{1}{2}-u} \;  \sum\limits_{j=1}^t  \phi_{t-j}^{u+\frac{1}{2}}  \; \paren{\Psi_{\xx}(\la_t) \cdot a_j 
+  \Theta_{\zz}(\la_t) } \;,\\
S_3(u,t)&:=  C^{(2)}_{\alpha, \lip} \; \eta^{1-u}   \sum\limits_{j=1}^t  \phi_{t-j}^{u} \; \; \sqrt{\la_t} \paren{\Psi_{\xx}(\la_t) \cdot a_j 
+  \Theta_{\zz}(\la_t) } \;.
\end{align*}
Let $u \in \{0, 1/2\}$, set $A_t := \sup_{j=1,...,t} a_j$ and $\la_t= \frac{1}{\eta (t+1)}$ and suppose that Assumption \ref{ass:eff.bound} 
holds true. From Assumption~\ref{ass:eta}, we have $\eta < \frac{1}{\kao^2}$. We apply 
Lemma \ref{Lemma:sum.phi}  and Lemma \ref{cor:prop-bounds} to 
bound the individual terms each with probability at least $1-\delta/4$ by 
\begin{align*}
S_2(s,t)
&\leq  C^{(1)}_{\alpha, \lip} \; \eta^{\frac{1}{2}-u}\;
         \sum\limits_{j=1}^t  \phi_{t-j}^{u+\frac{1}{2}} \; \paren{\Psi_{\xx}\left(\la_t\right)\cdot a_j+\Theta_{\zz}\left(\la_t \right)} \\
&\leq C^{(1)}_{\alpha, \lip} \;  \paren{C_{\ka} A_t+C_{\ka, M, \Sigma}} \; \eta^{\frac{1}{2}-u} \; \frac{1}{\sqrt{n\la_t^{\nu}}}   \;
     \log\left(\frac{6}{\delta}\right)\; 
     \sum\limits_{j=1}^t \phi_{t-j}^{u+\frac{1}{2}}  \;  \\
&\leq 9 \; C^{(1)}_{\alpha, \lip} \;  \paren{C_{\ka} A_t+C_{\ka, M, \Sigma}} \;  (\eta (t+1))^{\frac{1}{2}-u}\;  \frac{1}{\sqrt{n\la_t^{\nu}}} \;
        \; \log\left(\frac{6}{\delta}\right)  \\
&\leq 9 \; C^{(1)}_{\alpha, \lip} \;  C_{\nu}\paren{C_{\ka} A_t+C_{\ka, M, \Sigma}} \; 
        (\eta(t+1))^{\frac{1}{2}+\frac{\nu}{2}-u}\; \frac{\log\left(\frac{6}{\delta}\right)}{\sqrt n} \;,\\
& \\
S_3(u,t)
&\leq C^{(2)}_{\alpha, \lip} \;  \paren{C_{\ka} A_t+C_{\ka, M, \Sigma}} \; \eta^{1-u} \; \sqrt{\frac{\la_t^{1-\nu}}{ n}}  \;
     \log\left(\frac{6}{\delta}\right)\; 
     \sum\limits_{j=1}^t \phi_{t-j}^{u}  \;  \\
&\leq 2 \; C^{(2)}_{\alpha, \lip} \;  \paren{C_{\ka} A_t+C_{\ka, M, \Sigma}} \;   (\eta(t+1))^{1-u} \; \sqrt{\frac{\la_t^{1-\nu}}{ n}}  \;
     \log\left(\frac{6}{\delta}\right)\;  \\
&\leq 2 \; C^{(2)}_{\alpha, \lip} \;  C_{\nu}\paren{C_{\ka} A_t+C_{\ka, M, \Sigma}} \;  (\eta (t+1))^{\frac{1}{2}+ \frac{\nu}{2}-u}
\; \frac{\log\left(\frac{6}{\delta}\right)}{\sqrt n}\;.
\end{align*}

Using these bounds and \eqref{eq:main-SGD2} we arrive at 
\begin{align}
\label{eq:main-SGD3}
\norm{\tp^u \eto}_{\HH_1} 
&\leq D \eta^{-(u+r)}\phi_t^{u+r} + C_\diamond \cdot F_u(t) \cdot\frac{\log\left(\frac{6}{\delta}\right)}{\sqrt n}  
+ C^{(1)}_{\alpha, \lip} \; \eta^{\frac{1}{2}-u}  \sum\limits_{j=1}^t  \phi_{t-j}^{u+\frac{1}{2}} \; a_j\cdot b_j  \;,
\end{align}
where for $u \in \{0, 1/2\}$ 
\[ F_u(t):= \max\{1, A_t\} \cdot (\eta (t+1))^{\frac{1}{2}+ \frac{\nu}{2}-u} \]
and 
\[ C_\diamond := 9 \; C_{\nu}\cdot \left( C_{\ka} +  C_{\ka, M, \Sigma}\right) \cdot \paren{C^{(1)}_{\alpha, \lip}  +  C^{(2)}_{\alpha, \lip}}
      \;. \]

\vspace{0.3cm}

{\bf Step 2: Proof by induction.} In what follows, under the given assumptions, we prove the following {\bf claim}: For any $n \in \mbn$ sufficiently large and for $u \in \{0,1/2\}$, 
there exists an early stopping time $T_{n} \in \mbn$ that is increasing in $n \in \mbn$, such that with probability at least $1-\delta$ 
\begin{equation}
\label{eq:to-show-GD} 
 \norm{\tp^u e_{t}}_{\HH_1} \leq C_{\star }\; (\eta t)^{-(r + u)} \; \log\left(\frac{6}{\delta}\right)\;,
\end{equation} 
for all $t \leq T_n$ and for some $ C_{\star } >0$, depending on $u$, $\eta$, $M$, $\Sigma$, $\kappa_0$, $\lip$, $\alpha$, $C_R$, $\widetilde C_R$, $d$, $D$,  $C_\nu$.

Further, we prove that 
\begin{equation}\label{eq:to-show-At-GD}
    A_t\leq d,
\end{equation}
i.e., all the iterates up to $t$ lie in a ball of radius $d$ around $\fp$. We prove this claim by induction over $t \in [T_n]$.

\vspace{0.2cm}

{\bf Base Case (Initial Step):} By Assumption \ref{ass:source} and our assumption 
\begin{equation}\label{d.assump}
C_{\star} \eta^{-r} \log \frac{6}{\delta} \leq d,    
\end{equation}
we have for $\delta \in [\delta_0 , 1]$:  
\begin{align*}
\norm{g_1 - \fp}_{\HH_1} 
&= \norm{T^rg}_{\HH_1} 
\leq D \cdot \kao^{2r} 
\leq C_{\star }\; \eta ^{-r } \; \log\left(\frac{6}{\delta}\right) \leq d\;,
\end{align*} 
with $C_{\star }= \eta^r \cdot D \cdot \kao^{2r} $. 
Similarly, 
\begin{align*}
\norm{ \tp^{\frac{1}{2}}(g_1 - \fp)}_{\HH_1} &\leq \kao \norm{T^rg}_{\HH_1} \leq D \cdot \kao^{2r+1} \leq C_{\star }\; \eta ^{-(r+\frac{1}{2}) } \; \log\left(\frac{6}{\delta}\right)\;,
\end{align*} 
with $C_{\star }= \eta^{r+\frac{1}{2}} \cdot D \cdot \kao^{2r+1} $.

\vspace{0.2cm}

{\bf Inductive Step:} 
Assume now that \eqref{eq:to-show-GD} and \eqref{eq:to-show-At-GD} hold for $t$. By \eqref{eq:main-SGD3} we have
\begin{align*} 
 \norm{\tp^u \eto}_{\HH_1} 
&\leq \underbrace{D \eta^{-(u+r)}\phi_t^{u+r}}_{\tilde T_1} + 
      \underbrace{  C_\diamond \cdot F_u(t) \cdot\frac{\log\left(\frac{6}{\delta}\right)}{\sqrt n} }_{\tilde T_2}
+     \underbrace{C^{(1)}_{\alpha, \lip} \; \eta^{\frac{1}{2}-u}  
      \sum\limits_{j=1}^t  \phi_{t-j}^{u+\frac{1}{2}} \; a_j\cdot b_j}_{\tilde T_3}  \;.
\end{align*}
We are going to bound each term individually. From Lemma \ref{Lemma:sum.phi}  we immediately obtain for any 
$r>0$, $u \in \{0, 1/2\}$ the bound 
\begin{align}
\label{eq:T1-SGD}
\tilde T_1 &= D\eta^{-(u+r)} \phi_t^{u+r} \leq D \eta^{-(u+r)} \cdot \paren{\frac{u+r}{u+r+t}}^{u+r} 
\leq D  \; (\eta \cdot (t+1))^{-(u+r)}\;.
\end{align}

\vspace{0.2cm}

For the second term, note that 
\[  \frac{1}{\sqrt n} (\eta (t+1))^{\frac{1}{2}+ \frac{\nu}{2}-u} \leq (\eta (t+1))^{-(u+r)} \;\; 
   \iff \;\;  \eta (t+1) \leq n^{\frac{1}{2r+\nu+1}} \;.   \]
Hence, by setting $T_n := \floor{\eta^{-1} \;  n^{\frac{1}{2r+\nu+1}}}$, we obtain for all $t \in [T_n -1]$
\begin{align}
\label{eq:tilde-t2-SGD}
\tilde T_2 
&\leq C_\diamond \cdot \max\{1, A_t\}\cdot  (\eta (t+1))^{-(u+r)} \cdot \log\left(\frac{6}{\delta}\right) \nonumber \\
&\leq C_\diamond \cdot \max\{1, d\} \cdot  (\eta (t+1))^{-(u+r)} \cdot \log\left(\frac{6}{\delta}\right) \;.
\end{align}

\vspace{0.2cm}

For bounding the third term, we apply \eqref{eq:to-show-GD} Lemma \ref{Lemma:sum.phi} and our assumption on $d$ \eqref{d.assump} to find 
\begin{align}
\label{eq:tilde-t3}
\tilde T_3 
&\leq C_\star^2 C^{(1)}_{\alpha, \lip} \; \eta^{-2r-u}  \; \log^2\left(\frac{6}{\delta}\right) \; 
      \sum\limits_{j=1}^t \left( \frac{u+\frac{1}{2}}{u+\frac{1}{2}+t-j}\right)^{u+\frac{1}{2}} \; \left( \frac{1}{j} \right)^{2r +\frac{1}{2}} \nonumber \\
&\leq c'_{r,u}\; C_\star^2 C^{(1)}_{\alpha, \lip} \; \eta^{-r}  \; \log^2\left(\frac{6}{\delta}\right) \; 
\left(\frac{1}{\eta (t+1)} \right)^{r+u} \nonumber    \\
&\leq c'_{r,u}\; C_\star\; C^{(1)}_{\alpha, \lip} \; d\;  \log\left(\frac{6}{\delta}\right) \; 
\left(\frac{1}{\eta (t+1)} \right)^{r+u} \,,
\end{align}
for some $c'_{r,u} >0$. Combining \eqref{eq:T1-SGD}, \eqref{eq:tilde-t2-SGD} and \eqref{eq:tilde-t3} give for all $t \in [T_n -1]$ with probability at least $1-\delta$ 
\begin{align*} 
 \norm{\tp^u \eto}_{\HH_1} &\leq  \;  C_\star \cdot \log\left(\frac{6}{\delta}\right) \; 
\left(\frac{1}{\eta (t+1)} \right)^{r+u}\;,
\end{align*}
with 
\[   C_\star = \max\{\eta^{r+u}  D \kao^{2(r+u)},  D   +  C_\diamond \cdot \max\{1, d\} + c'_{r,u}\; C_\star\; C^{(1)}_{\alpha, \lip} \}  \;. \]

In particular, by the initial step and our inductive assumption, we find 
\begin{align*}
A_{t+1} &=\max_{j=1,..., t+1}||e_j||_{\cH_1} \\
&= \max\left\{ A_t , ||e_{t+1}||_{\cH_1} \right\} \\
&\leq \max\left\{d,  C_\star \cdot \log\left(\frac{6}{\delta}\right) \; 
\left(\frac{1}{\eta (t+1)} \right)^{r} \right\}\\
&\leq  \max\left\{d,  C_\star \cdot \log\left(\frac{6}{\delta}\right) \; 
\left(\frac{1}{\eta } \right)^{r} \right\}\\
&= d, 
\end{align*}
for all $t \in [T_n-1]$ and $\delta \in [\delta_0 , 1]$, with $\delta_0 = 6e^{-\frac{d}{C_\star}\eta^r}$.
\end{proof}


\section{Proofs Stochastic Gradient Descent}
\label{app:proofs-SGD}



\subsection{Preparatory Results}

To prove the main result we follow a classical bias-variance decomposition 

\begin{equation*}
\et :=\ft-\fp=\zt  + \EE[\et], 
\end{equation*}
where 
\begin{equation*}
\EE[\et]=\EE[\ft]-\fp
\end{equation*}
is the bias and 
\begin{equation*}
 \zt=\ft-\EE[\ft]
\end{equation*}
is the variance. 

We bound the bias and variance in Proposition~\ref{bias.bound} and Proposition~\ref{variance.bound}, respectively. 
The convergence rates of the SGD iterate~\eqref{fk.mini} are discussed in Theorem~\ref{Thm:SGD.mini}.

\vspace{0.2cm}

We begin with a Lemma that gives an unrolled expression for the SGD iterates.

\begin{lemma}
\label{lem:SGD-prelim}
Let $j \in \mbn$ and $\la_t >0$. For any $t \geq 1$ we have 
\begin{equation*}
\EE\sbrac{\fto-\fp}=(I-\eta\tp)^t (f_1-\fp)-\eta \sum\limits_{j=1}^t  (I-\eta\tp)^{t-j} \bp^* \EE\sbrac{v_j}-\eta\sum\limits_{j=1}^t  (I-\eta\tp)^{t-j} \EE\sbrac{w_j},
\end{equation*}

where 
\begin{align*}
v_j &= R_{f_j}^*\ip\brac{r(f_j)+(\lp+\la_t I)^{-1} \brac{\paren{\lx-\lp}\paren{A(f_j)-A(\fp)}-\sx^*\VE}}+(R_{f_j}^*-I)\bp(f_j-\fp)\;,\\
w_j &=  \la_t A'(f_j)^*(\lp+\la_t I)^{-1}\brac{ \paren{\lx-\lp}\paren{A(f_j)-A(\fp)} -\sx^*\VE} \;.
\end{align*}
\end{lemma}

\begin{proof}[Proof of Lemma \ref{lem:SGD-prelim}]
For the error term $\et=\ft-\fp$, we have from the definition of the Stochastic Gradient Descent algorithm~\eqref{fk.mini} and by denoting $\VE=\yy-\hat F(\fp)$:
\begin{equation*}
\eto= \et -\frac{\eta}{b}\sum\limits_{i=b(t-1)+1}^{bt}B_{\ft,j_i}^*\paren{F_{j_i}(\ft)-y_{j_i}},
\end{equation*}

By measurability of the SGD iterate $\ft$ with respect to the filtration $\FF_n$, we get the conditional expectation
\begin{align*}
\EE\sbrac{\eto|\FF_n}= & \et -\eta \paren{\sx A'(\ft)}^* \left((\sx \circ A)(\ft)-\yy\right).  
\end{align*}

Let $\lam_t >0$. Denoting 
\begin{align*}
r(\ft)=&A(\ft)-A(\fp)-A'(\fp)(\ft-\fp),\\ v_t=&R_{\ft}^*\ip\brac{r(\ft)+(\lp+\la_t I)^{-1} \brac{\paren{\lx-\lp}\paren{A(\ft)-A(\fp)}-\sx^*\VE}}\\
&+(R_{\ft}^*-I)\bp(\ft-\fp)\\ \intertext{and} 
w_t=&\la_t A'(\ft)^*(\lp+\la_t I)^{-1}\brac{ \paren{\lx-\lp}\paren{A(\ft)-A(\fp)} -\sx^*\VE},  
\end{align*}
we obtain by using Assumption~\ref{ass:A} (iii) the following recursion:

\begin{align*}
\EE\sbrac{\eto|\FF_n}
= &  (I-\eta \tp)\et -\eta \bp^* v_t -\eta w_t,
\end{align*}

Taking the full conditional expectation
\begin{equation*}
 \EE\sbrac{\eto}  =   (I-\eta \tp) \EE\sbrac{\et} -\eta \bp^*  \EE\sbrac{v_t} -\eta  \EE\sbrac{w_t}.
\end{equation*}

Applying the recursion repeatedly for the error $\et$, we obtain:

\begin{equation*}
\EE\sbrac{\eto}=(I-\eta\tp)^t e_1-\eta \sum\limits_{j=1}^t  (I-\eta\tp)^{t-j} \bp^* \EE\sbrac{v_j}-\eta\sum\limits_{j=1}^t  (I-\eta\tp)^{t-j} \EE\sbrac{w_j}.
\end{equation*}
\end{proof}

{\bf 

\subsubsection*{Bias Bound.} }

\begin{proposition}[Bias Bound]
\label{bias.bound}
Let Assumptions~\ref{ass:true}, \ref{ass:bounded}, \ref{ass:kernel}, \ref{ass:Frechet}, \ref{ass:eta}, \ref{ass:A}, \ref{ass:source}, \ref{ass:eff.bound} and
$\ft$ be the t-th SGD iterate \eqref{fk.mini}. 
Assume further that $u \in \{0, 1/2\}$, $\delta \in [\delta_0 , 1]$, 
where $\delta_0 = 6e^{-\frac{d}{C_\star}\eta^{\min\{r, 1/2\}}}$.
The constants $ C_{\star } \;, C^{(1)}_{\alpha, \lip} \;, \widetilde{C}_{\diamond} >0$ 
depend on  $u$, $\eta$, $\widetilde M$, $\kappa_0$, $\lip$, $\alpha$, $C_R$, $\widetilde C_R$, $d$, $D$,  $C_\nu$.   
With probability at least $1-\delta$, we have 
\[\norm{\tp^u \EE\sbrac{\eto}}_{\HH_1}
\leq \widetilde{C}_\diamond  \cdot \max\{1, A_t\} \cdot  (\eta  (t+1))^{-(u+r)}  \log\left(\frac{6}{\delta}\right)  
+ C^{(1)}_{\alpha, \lip}  \eta^{\frac{1}{2}-u}  \sum\limits_{j=1}^t \phi_{t-j}^{u+\frac{1}{2}} \EE\sbrac{\norm{e_j}^2}^{\frac{1}{2}}\EE\sbrac{\norm{\tp^{\frac{1}{2}}e_j}^2}^{\frac{1}{2}} , \]
for all $t \in [T_n]$, where  $A_t := \max_{j=1,...,t} \EE\sbrac{\norm{e_{j}}_{\HH_1}}^{\frac{1}{2}} $ and $T_n := \floor{\eta^{-1} \;  n^{\frac{1}{2\min\{r,1/2\}+\nu+1}}}$. 
In particular, $\ft \in \cB_d(\fp)\cap \cD(A)$ for all $t \in [T_n]$. The expectation is taken over the $b$-fold uniform distribution on $[n]$ at step $t+1$.
\end{proposition}

\begin{proof}[Proof of Proposition \ref{bias.bound}] 
From the source condition Assumption \ref{ass:source} 
we have $e_1=\fp - g_1 = \tp^r g$. Applying Lemma \ref{lem:SGD-prelim} and the equation \eqref{phi.bound} gives 
\begin{align}\label{eq.1.bias}
\norm{\tp^u \EE\sbrac{\eto}}_{\HH_1}  \leq & \norm{\tp^u(I-\eta\tp)^t\tp^r g}_{\HH_1} + \eta \sum\limits_{j=1}^t  \norm{\tp^u(I-\eta\tp)^{t-j}\bp^*} \cdot \EE\sbrac{\norm{v_j}_{\LL}}  \nonumber \\  
&+ \eta \sum\limits_{j=1}^t  \norm{\tp^u(I-\eta\tp)^{t-j}} \cdot\EE\sbrac{\norm{w_j}}_{\HH_1} \nonumber \\  
\leq & \underbrace{D\eta^{-(u+r)}\phi_t^{u+r}}_{T_1}+ \underbrace{\eta^{\frac{1}{2}-u} \sum\limits_{j=1}^t  \phi_{t-j}^{u+\frac{1}{2}}  \EE\sbrac{\norm{v_j}_{\LL}}}_{T_2} 
+ \underbrace{\eta^{1-u} \sum\limits_{j=1}^t  \phi_{t-j}^{u} \EE\sbrac{\norm{w_j}_{\HH_1}} }_{T_3}\;,
\end{align}
where $v_j$ and $w_j$ are defined in Lemma \ref{lem:SGD-prelim}. In what follows, we bound the 
individual parts $T_2$ and $T_3$ separately to simplify 
the proof. 

\vspace{0.3cm}

{\bf Bounding $T_2$.}  By the definition of $v_j$ and invoking the notation from Definition \ref{def:norms}, we find 
\begin{equation*}
\norm{v_j}\leq  \norm{ R_{f_j}}\cdot\paren{\norm{\ip r(f_j)}_{\LL} + \Psi_{\xx}(\la_t) \norm{A(f_j)-A(\fp)}+\Theta_{\zz}(\la_t)}+\norm{R_{f_j}-I}\cdot\norm{\bp e_j}_{\LL}.
\end{equation*}
By Assumptions \ref{ass:Frechet}, \ref{ass:A}, applying Lemma \ref{lem:taylor}, and \eqref{eq:lipschitz3} we obtain
\begin{align*}
\norm{v_j} \leq &\frac{\widetilde C_R \cdot C_R}{2} \cdot \norm{f_j-\fp}_{\HH_1} \cdot \norm{\bp\paren{f_j-\fp}}_{L^2}\nonumber \\
 &+  \frac{\widetilde C_R \cdot \lip }{1-\alpha}\cdot \Psi_{\xx}(\la_t) \cdot ||f_j-\fp||_{\cH_1} + \widetilde C_R \; \Theta_{\zz}(\la_t) + 
C_R \; ||f_j-\fp||_{\cH_1} \cdot \norm{\bp\paren{f_j-\fp}}_{L^2} \;,
\end{align*}
which implies
\begin{align*}
\EE\sbrac{\norm{v_j}}\leq & C^{(1)}_{\alpha, \lip}\paren{\EE\sbrac{\norm{f_j-\fp}_{\cH_1}^2}^{\frac{1}{2}} \cdot \EE\sbrac{\norm{\bp\paren{f_j-\fp}}_{L^2}^2}^{\frac{1}{2}}  + \Psi_{\xx} (\la_t)\cdot \EE\sbrac{\norm{f_j-\fp}_{\cH_1}^2}^{\frac{1}{2}} +\Theta_{\zz}(\la_t)},
\end{align*}
where 
\[ C^{(1)}_{\alpha, \lip} =  \max\left\{  \frac{\widetilde C_R \cdot C_R}{2} ,  \frac{\widetilde C_R \cdot \lip }{1-\alpha},  \widetilde C_R , C_R \right\}  \;. \]

Hence, setting 
\[ a_j :=  \EE\sbrac{\norm{f_j-\fp}_{\cH_1}^2}\;, \quad b_j:=\EE\sbrac{\norm{\bp\paren{f_j-\fp}}_{L^2}^2} \]
yields 
\begin{align}
\label{eq:T2}
T_2 &\leq C^{(1)}_{\alpha, \lip} \; \eta^{\frac{1}{2}-u} \sum\limits_{j=1}^t  \phi_{t-j}^{u+\frac{1}{2}} \left( a_j^{\frac{1}{2}}\cdot b_j^{\frac{1}{2}} 
+   \Psi_{\xx}(\la_t) \cdot a_j^{\frac{1}{2}} + \Theta_{\zz}(\la_t)  \right) \;. 
\end{align}
 
\vspace{0.2cm}

{\bf Bounding $T_3$.}  By the definition of $w_j$ we obtain by Assumption \ref{ass:Frechet}, by \eqref{eq:lipschitz3}
and Definition \ref{def:norms} 
\begin{align*}
\norm{w_j}_{\cH_1} 
&=  \la_t    \norm{ A'(f_j)^*(\lp+\la_t I)^{-1}\brac{ \paren{\lx-\lp}\paren{A(f_j)-A(\fp)} -\sx^*\VE } }_{\cH_1} \nonumber\\
&\leq \la_t \norm{ A'(f_j)^*(\lp+\la_t I)^{-1} \paren{\lx-\lp} \cdot \paren{A(f_j)-A(\fp)}}_{\cH_1} + 
     \la_t  \norm{A'(f_j)^*(\lp+\la_t I)^{-1} \sx^*\VE  }_{\cH_1} \nonumber\\
&\leq \la_t \norm{A'(f_j)^*(\lp+\la_t I)^{-\frac{1}{2}}} \cdot \norm{(\lp+\la_t I)^{-\frac{1}{2}} \paren{\lx-\lp}} \cdot ||A(f_j)-A(\fp)||_{\cH_2} \nonumber \\ 
& + \;\; \;\la_t \norm{A'(f_j)^*(\lp+\la_t I)^{-\frac{1}{2}}} \cdot \norm{(\lp+\la_t I)^{-\frac{1}{2}} \sx^*\VE  }_{\cH_1} \nonumber\\
&\leq  \lip \cdot \sqrt{\la_t} \cdot \Psi_{\xx}(\la_t) \cdot ||A(f_j)-A(\fp)||_{\cH_2} +  \lip \cdot \sqrt{\la_t} \cdot \Theta_{\zz}(\la_t) \nonumber \\
&\leq  \frac{\lip^2\sqrt{\la_t} }{1-\alpha}  \cdot \Psi_{\xx}(\la_t) \cdot ||f_j-\fp||_{\cH_1} +  \lip \cdot \sqrt{\la_t}\cdot \Theta_{\zz}(\la_t)  \;,
\end{align*}
which implies
\begin{equation*}
\EE\sbrac{\norm{w_j}}  \leq  \frac{\lip^2\sqrt{\la_t} }{1-\alpha}  \cdot \Psi_{\xx}(\la_t) \cdot \EE\sbrac{\norm{f_j-\fp}_{\cH_1}^2}^{\frac{1}{2}} +  \lip \cdot \sqrt{\la_t}\cdot \Theta_{\zz}(\la_t)  \;.
\end{equation*}

Hence, 
\begin{align}
\label{eq:T3}
T_3 &\leq  C^{(2)}_{\alpha, \lip} \; \eta^{1-u}  \; 
 \sum\limits_{j=1}^t  \phi_{t-j}^{u} \sqrt{\la_t} \cdot \left(  \Psi_{\xx}(\la_t) \cdot a_j^{\frac{1}{2}} 
+  \Theta_{\zz}(\la_t) \right)  \;,
\end{align}
where we set 
\[ C^{(2)}_{\alpha, \lip} = \max\left\{ \frac{\lip^2 }{(1-\alpha)} , \lip \right\}  \;. \]

\vspace{0.2cm}

Combining \eqref{eq.1.bias}, \eqref{eq:T2}, and \eqref{eq:T3}  yields 
\begin{align}
\label{eq:main-GD2}
\norm{\tp^u \EE\sbrac{\eto}}_{\HH_1} 
&\leq D \eta^{-(u+r)}\phi_t^{u+r} + 
C^{(1)}_{\alpha, \lip} \; \eta^{\frac{1}{2}-u} \sum\limits_{j=1}^t  \phi_{t-j}^{u+\frac{1}{2}} \left( a_j^{\frac{1}{2}}\cdot b_j^{\frac{1}{2}} 
+   \Psi_{\xx}(\la_t) \cdot a_j^{\frac{1}{2}} + \Theta_{\zz}(\la_t) \right)\nonumber  \\
&+ \;\;\;  C^{(2)}_{\alpha, \lip} \; \eta^{1-u}  \; \sum\limits_{j=1}^t  \phi_{t-j}^{u} \sqrt{\la_t} \cdot \left(  \Psi_{\xx}(\la_t) \cdot a_j^{\frac{1}{2}} 
+  \Theta_{\zz}(\la_t) \right) \nonumber \\
&=  D \eta^{-(u+r)}\phi_t^{u+r} +  \sum_{i=1}^3 S_i(u,t)\;,
\end{align}
where we set
\begin{align*}
S_1(u,t)&:= C^{(1)}_{\alpha, \lip} \; \eta^{\frac{1}{2}-u}  \sum\limits_{j=1}^t  \phi_{t-j}^{u+\frac{1}{2}} \; a_j^{\frac{1}{2}}\cdot b_j^{\frac{1}{2}}  \;,\\
S_2(u,t)&:=  C^{(1)}_{\alpha, \lip} \; \eta^{\frac{1}{2}-u} \;  \sum\limits_{j=1}^t  \phi_{t-j}^{u+\frac{1}{2}}  \; \paren{\Psi_{\xx}(\la_t) \cdot a_j^{\frac{1}{2}} 
+  \Theta_{\zz}(\la_t) } \;,\\
S_3(u,t)&:=  C^{(2)}_{\alpha, \lip} \; \eta^{1-u}   \sum\limits_{j=1}^t  \phi_{t-j}^{u} \; \; \sqrt{\la_t} \paren{\Psi_{\xx}(\la_t) \cdot a_j^{\frac{1}{2}} 
+  \Theta_{\zz}(\la_t) } \;.
\end{align*}
Let $u \in \{0, 1/2\}$, set $A_t := \max_{j=1,...,t} a_j^{\frac{1}{2}}$ and $\la_t= \frac{1}{\eta (t+1)}$ and suppose that Assumption \ref{ass:eff.bound} 
holds true. Form Assumption \ref{ass:eta} we have that $\eta < \frac{1}{\kao^2}$. We apply 
Lemma \ref{Lemma:sum.phi}  and Lemma \ref{cor:prop-bounds} to 
bound the individual terms each with probability at least $1-\delta/4$ by 
\begin{align*}
S_2(s,t)
&\leq  C^{(1)}_{\alpha, \lip} \; \eta^{\frac{1}{2}-u}\;
         \sum\limits_{j=1}^t  \phi_{t-j}^{u+\frac{1}{2}} \; \paren{\Psi_{\xx}\left(\la_t\right)\cdot a_j^{\frac{1}{2}}+\Theta_{\zz}\left(\la_t \right)} \\
&\leq C^{(1)}_{\alpha, \lip} \;  \paren{C_{\ka} A_t+C_{\ka, M, \Sigma}} \; \eta^{\frac{1}{2}-u} \; \frac{1}{\sqrt{n\la_t^{\nu}}}   \;
     \log\left(\frac{6}{\delta}\right)\; 
     \sum\limits_{j=1}^t \phi_{t-j}^{u+\frac{1}{2}}  \;  \\  
&\leq 9 \; C^{(1)}_{\alpha, \lip} \;  \paren{C_{\ka} A_t+C_{\ka, M, \Sigma}} \;  (\eta (t+1))^{\frac{1}{2}-u}\;  \frac{1}{\sqrt{n\la_t^{\nu}}}   \;
     \log\left(\frac{6}{\delta}\right)\;  \\
&\leq 9 \; C^{(1)}_{\alpha, \lip} \;  C_{\nu}\paren{C_{\ka} A_t+C_{\ka, M, \Sigma}} \; 
        (\eta(t+1))^{\frac{1}{2}+\frac{\nu}{2}-u}\; \frac{\log\left(\frac{6}{\delta}\right)}{\sqrt n} \;,\\
& \\
S_3(u,t)
&\leq C^{(2)}_{\alpha, \lip} \;  \paren{C_{\ka} A_t+C_{\ka, M, \Sigma}} \; \eta^{1-u} \; \sqrt{\frac{\la_t^{1-\nu}}{ n}}  \;
     \log\left(\frac{6}{\delta}\right)\; 
     \sum\limits_{j=1}^t \phi_{t-j}^{u}  \;  \\
&\leq 2 \; C^{(2)}_{\alpha, \lip} \;  \paren{C_{\ka} A_t+C_{\ka, M, \Sigma}} \;   (\eta(t+1))^{1-u} \; \sqrt{\frac{\la_t^{1-\nu}}{ n}}  \;
     \log\left(\frac{6}{\delta}\right)\;  \\     
&\leq 2 \; C^{(2)}_{\alpha, \lip} \;  C_{\nu}\paren{C_{\ka} A_t+C_{\ka, M, \Sigma}} \;  (\eta (t+1))^{\frac{1}{2}+ \frac{\nu}{2}-u}
\; \frac{\log\left(\frac{6}{\delta}\right)}{\sqrt n}\;.
\end{align*}

Using these bounds and \eqref{eq:main-GD2} we arrive at 
\begin{align}\label{eq:main-GD3}
 \norm{\tp^u \EE\sbrac{\eto}}_{\HH_1} 
&\leq \underbrace{D \eta^{-(u+r)}\phi_t^{u+r}}_{\tilde T_1} + 
      \underbrace{  C_\diamond \cdot F_u(t) \cdot\frac{\log\left(\frac{6}{\delta}\right)}{\sqrt n} }_{\tilde T_2}
+     C^{(1)}_{\alpha, \lip} \; \eta^{\frac{1}{2}-u}  
      \sum\limits_{j=1}^t  \phi_{t-j}^{u+\frac{1}{2}} \; a_j^{\frac{1}{2}}\cdot b_j^{\frac{1}{2}}  \;,
\end{align}

where for $u \in \{0, 1/2\}$ 
\[ F_u(t):= \max\{1, A_t\} \cdot (\eta (t+1))^{\frac{1}{2}+ \frac{\nu}{2}-u} \]
and 
\[ C_\diamond := 9 \; C_{\nu}\cdot \left( C_{\ka} +  C_{\ka, M, \Sigma}\right) \cdot \paren{C^{(1)}_{\alpha, \lip}  +  C^{(2)}_{\alpha, \lip}}
      \;. \]

\vspace{0.3cm}

We are going to bound each term individually. 

The fact
$$\paren{\frac{u+r}{u+r+t}}^{u+r}\leq 
\begin{cases}
\paren{\frac{1}{t+1}}^{u+r}; \quad\text{for} \quad u+r\leq 1\\
\paren{\frac{u+r}{t+1}}^{u+r}; \quad\text{for} \quad  1 \leq u+r
\end{cases}$$
implies that
\begin{equation}\label{phi.s.r.bias}
\phi_t^{u+r}\leq \max\paren{1,(u+r)^{u+r}}\paren{\frac{1}{t+1}}^{u+r}.  
\end{equation}

We immediately obtain for any 
$r>0$, $u \in \{0, 1/2\}$ the bound 
\begin{align}
\label{eq:T1}
\tilde T_1 &= D\eta^{-(u+r)} \phi_t^{u+r} \leq D \eta^{-(u+r)} \cdot \paren{\frac{u+r}{u+r+t}}^{u+r} 
\leq D  \; (\eta  (t+1))^{-(u+r)}\;.
\end{align}

\vspace{0.2cm}

For the second term, note that 
\[  \frac{1}{\sqrt n} (\eta (t+1))^{\frac{1}{2}+ \frac{\nu}{2}-u} \leq (\eta (t+1))^{-(u+r)} \;\; 
   \iff \;\;  \eta (t+1) \leq n^{\frac{1}{2r+\nu+1}} \;.   \]

Hence, by setting $T_n := \floor{\eta^{-1} \;  n^{\frac{1}{2r+\nu+1}}}$, we obtain for all $t \in [T_n -1]$
\begin{align}
\label{eq:tilde-t2}
\tilde T_2 
&\leq C_\diamond \cdot \max\{1, A_t\} \cdot  (\eta (t+1))^{-(u+r)} \cdot \log\left(\frac{6}{\delta}\right) \;.
\end{align}

\vspace{0.2cm}

Using these bounds and \eqref{eq:main-GD3} we arrive at  
\begin{align}
\norm{\tp^u \EE\sbrac{\eto}}_{\HH_1}
&\leq \widetilde{C}_\diamond \cdot \max\{1, A_t\} \cdot  (\eta  (t+1))^{-(u+r)}\; \cdot \log\left(\frac{6}{\delta}\right)  
+ C^{(1)}_{\alpha, \lip} \; \eta^{\frac{1}{2}-u}  \sum\limits_{j=1}^t  \phi_{t-j}^{u+\frac{1}{2}} \; a_j^{\frac{1}{2}}\cdot b_j^{\frac{1}{2}}  \;,
\end{align}
where $\widetilde{C}_\diamond=D+C_\diamond $.
\end{proof}


\subsubsection*{Variance Bound.}

\vspace{0.2cm}

\begin{proposition}[Variance Bound]
\label{variance.bound}
Let Assumptions~\ref{ass:true}, \ref{ass:bounded}, \ref{ass:kernel}, \ref{ass:Frechet}, \ref{ass:A},  \ref{ass:eff.bound}, \ref{ass:eta} and
$\zt$ is the variance term for SGD iterate. 
Assume further that $u \in \{0, 1/2\}$,  $\delta \in (0, 1]$.
The constant $C_{\al,\lip}^{(8)}$ 
depends on  $u, C_R, \widetilde C_R , C_\nu,  d, D,\alpha, \lip, \ka, M, \Sigma$.   
With probability at least $1-\delta$, we have 
\begin{equation*}
\EE\sbrac{\norm{\tp^u \zto}^2} \leq  C_{\al,\lip}^{(8)}\cdot (1+A_t^2)^2 \cdot  (\eta(t+1))^{-2(r+u)} \log^2\paren{\frac{2}{\delta}}\brac{1+\eta^{\vartheta}\sum\limits_{j=1}^{t}\phi_{t-j}^1\EE\sbrac{\norm{e_j}^2}}, 
\end{equation*}
provided that
\begin{equation}\label{b.bound}
b\geq \max\paren{\eta (\eta t)^{2r+\nu},\eta^{\frac{2r+1}{2r+\nu+1}} (\eta t)^{2r+1}} 
\end{equation}
for all $t \in [T_n]$, where  $A_t := \max_{j=1,...,t} \EE\sbrac{\norm{e_{j}}_{\HH_1}}^{\frac{1}{2}} $ and $T_n := \floor{\eta^{-1} \;  n^{\frac{1}{2\min\{r,1/2\}+\nu+1}}}$. 
In particular, $\ft \in \cB_d(\fp)\cap \cD(A)$ for all $t \in [T_n]$. The expectation is taken over the $b$-fold uniform distribution on $[n]$ at step $t+1$.
\end{proposition}
  
\begin{proof}[Proof of Proposition \ref{variance.bound}]
By the definition of the SGD iterate $\fto$ in \eqref{fk.mini}, we have
\begin{equation*}
\fto= \ft-\frac{\eta}{b}\sum\limits_{i=b(t-1)+1}^{bt}B_{\ft,j_i}^*\paren{F_{j_i}(\ft)-y_{j_i}},
\end{equation*}
which can be expressed as
\begin{equation}\label{fk.mini.exp}
\fto= \ft-\frac{\eta}{b}\sum\limits_{i=b(t-1)+1}^{bt}\paren{B_{j_i}^*B_{j_i}\et+B_{j_i}^*\epsilon_{j_i}+B_{j_i}^*v_{t,j_i}+w_{t,j_i}}
\end{equation}
where the random variables $\epsilon_{j_i}=F_{j_i}(\fp)-y_{j_i}$, $v_{t,j_i}=F_{j_i}(\ft)-F_{j_i}(\fp)-B_{j_i}(\ft-\fp)$ and $w_{t,j_i}=\paren{B_{\ft,j_i}^*- B_{j_i}^*}\paren{F_{j_i}(\ft)-y_{j_i}}$.

By measurability of the iterate $\ft$ with respect to the filtration $\FF_n$, we get the conditional expectation:
\begin{align*}
\EE\sbrac{\fto|\FF_n} = & \ft- \frac{\eta}{n}\sum\limits_{i=1}^n B_i^* B_i \et -\frac{\eta}{n} \sum\limits_{i=1}^n B_i^* \epsilon_i -\frac{\eta}{n} \sum\limits_{i=1}^n B_i^* v_{t,i} -\frac{\eta}{n} \sum\limits_{i=1}^n  w_{t,i}\\
= & \ft - \eta \tx \et - \eta \bx^* \VE - \eta \bx^* v_t -\eta w_t,
\end{align*}
where $(\VE)_i=\epsilon_i$, $v_t= \sx r(\ft)$ and $w_t= \paren{\hat{B}_{\ft}^*- \bx^*}\paren{\hat{F}(\ft)-\yy}$.

Taking full conditional yields
\begin{equation}\label{Efk.mini.exp}
\EE\sbrac{\fto}=\EE\sbrac{\ft}-\eta \paren{\tx \EE\sbrac{\et} +\bx^*\VE+ \bx^*\EE\sbrac{v_t}+ \EE\sbrac{w_t}}.
\end{equation}

Thus, subtracting the recursion for $\EE(\fto)$ in \eqref{Efk.mini.exp} from $\fto$ in \eqref{fk.mini.exp} we get the random variable $\zto:=\fto-\EE(\fto)$ satisfies
\begin{align}\label{zk.mini}
\zto
= & (I-\eta \tx)\zt+\eta M_t,
\end{align}

where 
\begin{align*}
M_t:= &  M_{t,1} +  M_{t,2} +  M_{t,3} +  M_{t,4},
\end{align*}
and
\begin{align*}
M_{t,1}= & \frac{1}{b}\sum\limits_{i=b(t-1)+1}^{bt}\paren{\tx -T_{j_i}}\et,\\
M_{t,2}= & \frac{1}{b}\sum\limits_{i=b(t-1)+1}^{bt}\paren{\bx^*\VE -  B_{j_i}^*\epsilon_{j_i}},\\
M_{t,3}= & \frac{1}{b}\sum\limits_{i=b(t-1)+1}^{bt}\paren{\bx^*\EE(v_t)-B_{j_i}^*v_{t,j_i}},\\
M_{t,4}= & \frac{1}{b}\sum\limits_{i=b(t-1)+1}^{bt}\paren{\EE(w_t)-w_{t,j_i}}.
\end{align*} 

The random variable $M_t$ represents the iteration noise, due to the random choice of the index $j_i$. This splitting enables separately treating conditionally independent and dependent factors. Repeatedly applying the recursion \eqref{zk.mini} and using the initial condition $z_1 = 0$ lead to
\begin{equation*}
\zto =\eta \sum\limits_{j=1}^t(I-\eta\tx)^{t-j} M_j.
\end{equation*}

For $Q_{t-j}=(I-\eta\tx)^{t-j}$, we have
\begin{align*}
\zto =& \eta\sum_{i=1}^4 \sum\limits_{j=1}^t  Q_{t-j} M_{j,i}.
\end{align*}

\begin{align*}
\EE\sbrac{\norm{\tp^u \zto}^2} \leq & 4 \eta^2\sum\limits_{i=1}^4\EE\sbrac{\norm{\sum\limits_{j=1}^{t} \tp^{u}Q_{t-j} M_{j,i}}^2}\\
= & 4 \eta^2\sum\limits_{i=1}^4\sum\limits_{j=1}^{t}  \EE\sbrac{\norm{\tp^{u}Q_{t-j} M_{j,i}}^2}\\
= & 4 \eta^2\sum\limits_{i=1}^4\sum\limits_{j=1}^{t}  \EE\sbrac{\tr\paren{\tp^{2u}Q_{t-j}^2 M_{j,i}\otimes M_{j,i}}} \\
= & 4 \eta^2 \sum\limits_{i=1}^4\sum\limits_{j=1}^{t}  \tr\paren{\tp^{2u}Q_{t-j}^2 \EE\sbrac{M_{j,i}\otimes M_{j,i}}}.
\end{align*}

By setting $A_{t} =\max_{j=1,..., t} \EE\sbrac{\norm{e_j}^2}^{\frac{1}{2}}$ and using the estimates of Lemma~\ref{Lem.Mi.Mi} we get:
\begin{align}\label{tp.zto}
&\EE\sbrac{\norm{\tp^u \zto}^2} \\  \nonumber
\leq &\frac{\eta^2C_{\al,\lip}^{(5)}}{b} \cdot (1+A_t^2) \cdot\sum\limits_{j=1}^{t}\Bigg\{\tr\sbrac{\tp^{2u}Q_{t-j}^2\tx}\\    \nonumber
&+\paren{\tr\sbrac{\tp^{2u+1}Q_{t-j}^2}+\Psi_{\xx}\cdot\tr\sbrac{\tp^{2u+\frac{1}{2}}Q_{t-j}^2} +\la\Upsilon_{\xx} \cdot\norm{\tp^{2u}Q_{t-j}^2 }}\EE\sbrac{\norm{e_j}^2}\Bigg\} ,
\end{align}
where $C_{\al,\lip}^{(5)}=4\brac{C_{\al,\lip}^{(4)} +\Sigma^2+5 \kao^2 }$.

\vspace{0.2cm}

{\bf  Case $u=0$:}

\vspace{0.2cm}

For the case $u=0$, we get
\begin{align}\label{tp.zt.0}
&\EE\sbrac{\norm{ \zto}^2} \\  \nonumber
\leq &\frac{\eta^2C_{\al,\lip}^{(5)}}{b} \cdot (1+A_t^2) \cdot \sum\limits_{j=1}^{t}\Bigg\{\tr\sbrac{\tx^\nu}\cdot\norm{Q_{t-j}^2\tx^{1-\nu}}\\    \nonumber
&+\paren{\tr\sbrac{\tp^\nu}\cdot\norm{\tp^{1-\nu}Q_{t-j}^2}+\Psi_{\xx}\cdot\tr\sbrac{\tp^\nu}\cdot\norm{\tp^{\frac{1}{2}-\nu}Q_{t-j}^2} +\la\Upsilon_{\xx} \cdot\norm{Q_{t-j}^2 }}\EE\sbrac{\norm{e_j}^2}\Bigg\}.
\end{align}

For $v\leq 1$, we have
\begin{align}\label{tp.Q}
\norm{\tp^{v} Q_{t-j}^2}\leq &  \norm{(\tp+\la I)^{v}(\tx+\la I)^{-v}}\norm{(\tx+\la I)^{v} Q_{t-j}^2}\nonumber \\
    \leq &  \Xi^{v}\brac{\norm{\tx^{v}Q_{t-j}^2}+\la^{v}\norm{Q_{t-j}^2}}.
\end{align}

By using the bounds \eqref{tp.Q} in \eqref{tp.zto} with $A_{t} =\max_{j=1,..., t} \EE\sbrac{\norm{e_j}^2}^{\frac{1}{2}} $, we obtain
\begin{align}\label{tx.ak.m44}
&\EE\sbrac{\norm{\zto}^2} \\  \nonumber
\leq &\frac{\eta^2C_{\al,\lip}^{(5)}}{b}  \cdot (1+A_t^2)^2 \cdot\sum\limits_{j=1}^{t}\brac{\norm{\tx^{1-\nu} Q_{t-j}^2}+\la^{\frac{1}{2}}\norm{\tx^{ \frac{1}{2}-\nu}Q_{t-j}^2}+\la^{1-\nu}\norm{Q_{t-j}^2}}\\    \nonumber
&\cdot\paren{\tr\sbrac{\tp^\nu}\paren{\Xi^{1-\nu}+1}+\tr\sbrac{\tp^\nu}\Xi^{\frac{1}{2}-\nu}\frac{\Psi_\xx}{\sqrt{\la}}+\Upsilon_{\xx} }.
\end{align}

For $v\leq 1$, the inequality $\paren{\frac{v}{v+t-j}}^{v}  \leq \paren{\frac{1}{1+j}}^{v}$ yields
\begin{align}
\norm{(\eta\tx)^{2v}Q_{t-j}^2} \leq \paren{\frac{v}{v+t-j}}^{2v} \leq \paren{\frac{1}{1+t-j}}^{2v} = \paren{\phi_{t-j}^1}^{2v}.
\end{align}
Using the above inequality in \eqref{tx.ak.m44} and Lemma~\ref{cor:prop-bounds}, we obtain with confidence $1-\delta$:
\begin{align}\label{tx.ak.m44.0}
&\EE\sbrac{\norm{ \zto}^2} \\  \nonumber
\leq &\frac{\eta^2C_{\al,\lip}^{(6)}}{b} \cdot \frac{\nu}{5}\cdot (1+A_t^2)^2 \cdot \sum\limits_{j=1}^{t}\brac{\paren{\eta^{\nu-1}\paren{\phi_{t-j}^1}^{1-\nu}+\la^{\frac{1}{2}}\eta^{\nu-\frac{1}{2}}\paren{\phi_{t-j}^{1}}^{\frac{1}{2}-\nu}+\la^{1-\nu}}}\log^2\paren{\frac{2}{\delta}},
\end{align}
where $C_{\al,\lip}^{(6)}=C_{\al,\lip}^{(5)}\cdot \brac{\tr\sbrac{\tp^\nu}\paren{2 C_{\ka}^2+1}+C_{\ka}}\cdot \frac{5}{\nu}.$

Using Lemma~\ref{Lemma:sum.phi} we conclude that with confidence $1-\delta$:
\begin{align*}
&\EE\sbrac{\norm{ \zto}^2} \\
\leq &\frac{\eta^2C_{\al,\lip}^{(6)}}{b} \cdot \frac{\nu}{5} \cdot (1+A_t^2)^2 \cdot\Bigg\{& \frac{2}{\nu}\eta^{\nu-1}(t+1)^{\nu}+ \frac{2}{\nu}\la^{\frac{1}{2}}\eta^{\nu-\frac{1}{2}}(t+1)^{\nu+\frac{1}{2}}+\la^{1-\nu}(t+1)\Bigg\}\log^2\paren{\frac{2}{\delta}}.
\end{align*}

For the choice $\la=\frac{1}{\eta(t+1)}$, we obtain with confidence $1-\delta$:
\begin{align*}
\EE\sbrac{\norm{ \zto}^2} \leq  \frac{C_{\al,\lip}^{(6)}}{b}\cdot (1+A_t^2)^2\cdot \eta^{\nu+1}(t+1)^{\nu}\log^2\paren{\frac{2}{\delta}}.
\end{align*}

Under the condition \eqref{b.bound}, we conclude with the probability $1-\delta$:
\begin{align}\label{tx.bar.z}
\EE\sbrac{\norm{ \zto}^2} \leq  C_{\al,\lip}^{(6)}\cdot (1+A_t^2)^2\cdot (\eta(t+1))^{-2r}\log^2\paren{\frac{2}{\delta}}.
\end{align}
\vspace{0.2cm}

{\bf  Case $u=\frac{1}{2}$:}

For the case $u=\frac{1}{2}$ and $\vartheta=\frac{\nu}{2r+\nu+1}$, we get
\begin{align}
&\EE\sbrac{\norm{\tp^{\frac{1}{2}} \zto}^2} \\  \nonumber
\leq &\frac{\eta^2C_{\al,\lip}^{(5)}}{b} \cdot (1+A_t^2)\cdot \sum\limits_{j=1}^{t}\Bigg\{\tr\sbrac{\tp^{1-\vartheta}}\norm{\tp^{\vartheta}Q_{t-j}^2\tx}\\    \nonumber
&+\paren{\tr\sbrac{T}+\Psi_\xx \tr\sbrac{T^{\frac{1}{2}}}+\la \Upsilon_{\xx}}\norm{\tp Q_{t-j}^2}\EE\sbrac{\norm{e_j}^2}
\Bigg\}.
\end{align}

Using the inequality \eqref{tp.Q}, we get 
\begin{align*}
&\EE\sbrac{\norm{\tp^{\frac{1}{2}} \zto}^2} \\  \nonumber
\leq &\frac{\eta^2C_{\al,\lip}^{(5)}}{b} \cdot (1+A_t^2) \cdot\sum\limits_{j=1}^{t}\Bigg\{\tr\sbrac{\tp^{1-\vartheta}}\Xi^{\vartheta}\brac{\norm{\tx^{1+\vartheta}Q_{t-j}^2}+\la^{\vartheta}\norm{\tx Q_{t-j}^2}}\Bigg\} \\    \nonumber
&+\paren{\tr\sbrac{T}+\Psi_\xx \tr\sbrac{T^{\frac{1}{2}}}+\la \Upsilon_{\xx}}\Xi\brac{\norm{\tx Q_{t-j}^2}+\la\norm{Q_{t-j}^2}}\EE\sbrac{\norm{e_j}^2}.
\end{align*}

Using Lemma~\ref{cor:prop-bounds} for $\la=\frac{1}{\eta(t+1)}$,  we conclude that with confidence $1-\delta$:
\begin{align}
&\EE\sbrac{\norm{\tp^{\frac{1}{2}} \zto}^2} \nonumber\\  
\leq &\frac{\eta^2C_{\al,\lip}^{(5)}}{b}  \cdot (1+A_t^2) \cdot\sum\limits_{j=1}^{t}\Bigg\{ \kao^{2(1-\vartheta)}\paren{C_{\ka}\log\paren{\frac{6}{\delta}}}^{2\vartheta}\paren{\eta^{-1-\vartheta}\paren{\phi_{t-j}^{1}}^{1+\vartheta}+\la^{\vartheta}\eta^{-1}\phi_{t-j}^{1}} \nonumber\\  
&+\paren{\kao^2+\kao C_{\ka}\sqrt{\la}\log\paren{\frac{6}{\delta}}+C_{\ka}\la \log\paren{\frac{6}{\delta}}}C_{\ka}^2\log^{2}\paren{\frac{6}{\delta}}\paren{\eta^{-1}\phi_{t-j}^{1}+\la}\EE\sbrac{\norm{e_j}^2}
\Bigg\}.
\end{align}

By setting $A_{t} =\max_{j=1,..., t} \EE\sbrac{\norm{e_{j}}_{\HH_1}}^{\frac{1}{2}}  $ and for sufficiently large $n$ such that $\log\paren{\frac{6}{\delta}} \leq (\eta(t+1))^{\frac{1}{2}}=n^{\frac{1}{2(2r+\nu+1)}}$ using Lemma~\ref{Lemma:sum.phi} we conclude that with confidence $1-\delta$:
\begin{align}\label{tx.bar.z1}
\EE\sbrac{\norm{\tp^{\frac{1}{2}} \zto}^2}  \leq  \frac{C_{\al,\lip}^{(7)}}{b} \cdot (1+A_t^2)^2 \cdot \eta^{1-\vartheta}\log^2\paren{\frac{6}{\delta}}\brac{1+\eta^\vartheta\sum\limits_{j=1}^{t}\phi_{t-j}^1\EE\sbrac{\norm{e_j}^2}},
\end{align}
where $C_{\al,\lip}^{(7)}=C_{\al,\lip}^{(5)}\max\brac{C_{\ka}^{2}\paren{\kao^{2}+\kao C_{\ka}+C_{\ka}},C_{\ka}^{2\vartheta}\kao^{2(1-\vartheta)}\paren{(1+\vartheta)\vartheta^{-1}+2(e\vartheta)^{-1}}}$.

For the choice of $b$ in \eqref{b.bound} and $\vartheta=\frac{\nu}{2r+\nu+1}$, we get with confidence $1-\delta$:
\begin{equation*}
\EE\sbrac{\norm{\tp^{\frac{1}{2}} \zto}^2} \leq  C_{\al,\lip}^{(7)}\cdot (1+A_t^2)^2 \cdot (\eta(t+1))^{-(2r+1)} \log^2\paren{\frac{2}{\delta}}\brac{1+\eta^{\vartheta}\sum\limits_{j=1}^{t}\phi_{t-j}^1\EE\sbrac{\norm{e_j}^2}}.
\end{equation*}

Hence, for $u\in\brac{0,\frac{1}{2}}$ we conclude with confidence $1-\delta$:
\begin{equation*}
\EE\sbrac{\norm{\tp^u \zto}^2} \leq  C_{\al,\lip}^{(8)}\cdot (1+A_t^2)^2 \cdot  (\eta(t+1))^{-2(r+u)} \log^2\paren{\frac{2}{\delta}}\brac{1+\eta^{\vartheta}\sum\limits_{j=1}^{t}\phi_{t-j}^1\EE\sbrac{\norm{e_j}^2}} , 
\end{equation*}
where $C_{\al,\lip}^{(8)}=\max\brac{C_{\al,\lip}^{(6)}, C_{\al,\lip}^{(7)}}$.
\end{proof}


\subsection{Proof of Theorem \ref{Thm:SGD.mini}}

\begin{proof}[Proof of Theorem \ref{Thm:SGD.mini}]
For the SGD iterate $\ft$ \eqref{fk.mini}, we have the standard bias-variance decomposition 
\begin{equation*}
 \eto=   \fto-\fp =\zto + \EE(\eto),
\end{equation*}
which implies
\begin{equation*}
  \EE\sbrac{ \norm{\tp^u \eto}^2} =\EE\sbrac{ \norm{\tp^u \zto}^2}  + \norm{\tp^u\EE(\eto)}^2.
\end{equation*}
Let $a_j=\EE\sbrac{\norm{e_j}^2}$ and $b_j=\EE\sbrac{\norm{\tp^{\frac{1}{2}}e_j}^2}$.  Now, by applying the bounds of Propositions~\ref{bias.bound}, ~\ref{variance.bound} in the above inequality, we conclude that
\begin{align}\label{tu.eto.ind}
\EE\sbrac{ \norm{\tp^u \eto}^2} \leq & C_{\al,\lip}^{(9)} \cdot (1+A_t^2)^2 \cdot \Bigg\{\paren{\eta(t+1)}^{-2(u+r)}\paren{2+\eta^{\vartheta}\sum\limits_{j=1}^t \phi_{t-j}^{1} a_j}\Bigg\}\log^2\paren{\frac{6}{\delta}}\nonumber\\
&+2\paren{C_{\al,\lip}^{(1)}}^2\eta^{1-2u}\paren{\sum\limits_{j=1}^t \phi_{t-j}^{u+\frac{1}{2}} a_j^{\frac{1}{2}}b_j^{\frac{1}{2}}
 }^2,
\end{align} 
where $C_{\al,\lip}^{(9)}=2\max\paren{2\widetilde{C}_{\diamond}^2, C_{\al,\lip}^{(8)}}$.

{\bf Proof by induction.} In what follows, under the given assumptions, we prove the following {\bf claim}: For any $n \in \mbn$ sufficiently large and for $u \in \{0,1/2\}$, 
there exists an early stopping time $T_{n} \in \mbn$ that is increasing in $n \in \mbn$, such that with probability at least $1-\delta$ 
\begin{equation}
\label{eq:to-show-SGD} 
\EE\sbrac{ \norm{\tp^u \et}^2}  \leq C_{\star }^2\; (\eta t)^{-2(r + u)} \; \log^2\left(\frac{6}{\delta}\right)\;,
\end{equation} 
for all $t \leq T_n$ and for some $ C_{\star } >0$, depending on $u$, $\eta$, $\widetilde M$, $\kappa_0$, $\lip$, $\alpha$, $C_R$, $\widetilde C_R$, $d$, $D$,  $C_\nu$. 

Further, we prove that 
\begin{equation}\label{eq:to-show-At-SGD}
    A_t\leq d,
\end{equation}
i.e., all the iterates up to $t$ lie in a ball of radius $d$ around $\fp$. We prove this claim by induction over $t \in [T_n]$. 

\vspace{0.2cm}

{\bf Base Case (Initial Step):} By Assumption \ref{ass:source} and our assumption 
\begin{equation}\label{d.assump.1}
C_{\star} \eta^{-r} \log \frac{6}{\delta} \leq d,    
\end{equation}
we have for $\delta \in [\delta_0 , 1]$:  
\begin{align*}
\EE\sbrac{ \norm{f_1 - \fp}_{\HH_1}^2}^{\frac{1}{2}}
&= \EE\sbrac{ \norm{T^rg}_{\HH_1}^2}^{\frac{1}{2}}
\leq D \cdot \kao^{2r} 
\leq C_{\star }\; \eta ^{-r } \; \log\left(\frac{6}{\delta}\right)\leq d\;,
\end{align*} 
with $C_{\star }= \eta^r \cdot D \cdot \kao^{2r} $. 

Similarly, 
\begin{align*}
\EE\sbrac{ \norm{\tp^{\frac{1}{2}}(f_1 - \fp)}_{\HH_1}^2}^{\frac{1}{2}} &\leq \kao  \EE\sbrac{ \norm{T^rg}_{\HH_1}^2}^{\frac{1}{2}} \leq D \cdot \kao^{2r+1} \leq C_{\star }\; \eta ^{-(r+\frac{1}{2}) } \; \log\left(\frac{6}{\delta}\right)\;,
\end{align*} 
with $C_{\star }= \eta^{r+\frac{1}{2}} \cdot D \cdot \kao^{2r+1} $.

\vspace{0.2cm}

{\bf Inductive Step:} 
Now we assume that the bounds \eqref{eq:to-show-SGD} and \eqref{eq:to-show-At-SGD} hold up to the case $t$, and prove the assertion for the case $t + 1$.

Using the bounds \eqref{eq:to-show-SGD} and \eqref{eq:to-show-At-SGD} in \eqref{tu.eto.ind}, we get with confidence $1-\delta$:
\begin{align*}
\EE\sbrac{ \norm{\tp^u \eto}^2} \leq & C_{\al,\lip}^{(9)} \cdot (1+d^2)^2 \cdot  \Bigg\{\paren{\eta(t+1)}^{-2(u+r)}\paren{2+C_{\star }\eta^{\vartheta}\sum\limits_{j=1}^t \phi_{t-j}^{1} (\eta j)^{-2r}\log^2\paren{\frac{6}{\delta}}}\Bigg\}\log^2\paren{\frac{6}{\delta}}\nonumber\\
&+2\paren{C_{\al,\lip}^{(1)}}^2C_{\star }^2\eta^{1-2u}\paren{\sum\limits_{j=1}^t \phi_{t-j}^{u+\frac{1}{2}} (\eta j)^{-2r-\frac{1}{2}}\log^2\paren{\frac{6}{\delta}}
 }^2.
\end{align*}

Using the Lemma~\ref{Lemma:sum.phi}, we get with confidence $1-\delta$:
\begin{align*}
\EE\sbrac{ \norm{\tp^u \eto}^2} \leq & C_{\al,\lip}^{(9)} \cdot (1+d^2)^2 \cdot \Bigg\{\paren{\eta(t+1)}^{-2(u+r)}\paren{2+C_{\star }c_r\eta^{\vartheta-2r} (t+1)^{-r}\log^2\paren{\frac{6}{\delta}}}\Bigg\}\log^2\paren{\frac{6}{\delta}}\nonumber\\
&+2\paren{C_{\al,\lip}^{(1)}}^2C_{\star }^2\eta^{\frac{1}{2}-2u-2r}\paren{ (t+1)^{-\frac{3r}{2}-u}\log^2\paren{\frac{6}{\delta}}
 }^2.
\end{align*}

Under the condition \eqref{d.assump.1}, we get with confidence $1-\delta$:
\begin{equation*}
\EE\sbrac{ \norm{\tp^u \eto}^2}  \leq C_{\star }\; (\eta (t+1))^{-2(r + u)} \; \log^2\left(\frac{6}{\delta}\right)\;.
\end{equation*}

In particular, by the initial step and our inductive assumption, we find 
\begin{align*}
A_{t+1} &=\max_{j=1,..., t+1}\EE\sbrac{\norm{e_{j}}_{\HH_1}}^{\frac{1}{2}} \\
&= \max\left\{ A_t , \EE\sbrac{\norm{e_{t+1}}_{\HH_1}}^{\frac{1}{2}}  \right\} \\
&\leq \max\left\{d,  C_\star \cdot \log\left(\frac{6}{\delta}\right) \; 
\left(\frac{1}{\eta (t+1)} \right)^{r} \right\}\\
&\leq  \max\left\{d,  C_\star \cdot \log\left(\frac{6}{\delta}\right) \; 
\left(\frac{1}{\eta } \right)^{r} \right\}= d, 
\end{align*}
for all $t \in [T_n-1]$ and $\delta \in [\delta_0 , 1]$, with $\delta_0 = 6e^{-\frac{d}{C_\star}\eta^r}$.
\end{proof}

The following two Lemmas are used to bound the variance term in Proposition~\ref{variance.bound}.

\begin{lemma}\label{Lem.Mi}
Let Assumptions~\ref{ass:bounded}, \ref{ass:kernel}, \ref{ass:Frechet}, and \ref{ass:A} hold true. For any $t\in\NN$, we have the following almost surely:
\begin{align*}
\EE\sbrac{M_{t,1}\otimes M_{t,1}}\preceq  & \frac{\kao^2}{b}\EE\sbrac{\norm{\et}^2}  \tx,\\
\EE\sbrac{M_{t,2}\otimes M_{t,2}}\preceq  & \frac{\Sigma^2}{b} \tx,\\
\EE\sbrac{M_{t,3}\otimes M_{t,3}}\preceq  & \frac{4 \kao^2}{b}\EE\sbrac{\norm{\et}^2} \tx,\\
\EE\sbrac{M_{t,4}\otimes M_{t,4}}\preceq  & \frac{C^{(3)}_{\alpha, \lip}}{b} \EE\sbrac{\paren{\norm{\et}^2+1}\paren{A'(\ft)-A'(\fp)}^*\lx\paren{A'(\ft)-A'(\fp)}},
\end{align*}
where $\Sigma = 2 \max\paren{\widetilde{M},\ka \norm{A(\fp)}}$ and $C^{(3)}_{\alpha, \lip} =  2 \max\left\{  \paren{\frac{\kao }{1-\alpha}}^2 , \Sigma^2 \right\}$. The expectation is taken over the $b$-fold uniform distribution on $[n]$ at step $t+1$.
\end{lemma}
\begin{proof}
We recall that
\begin{align*}
    M_{t,1}=&\frac{1}{b}\sum\limits_{i=b(t-1)+1}^{bt}(\tx-T_{j_i})\et=\frac{1}{b}\sum\limits_{i=b(t-1)+1}^{bt}\xi_{i,1},\\
M_{t,2} = & \frac{1}{b}\sum\limits_{i=bt+1}^{bt} \paren{\bx^*\VE-B_{j_i}^*\epsilon_{j_i}}=\frac{1}{b}\sum\limits_{i=bt+1}^{bt} \xi_{i,2},\\
M_{t,3} = & \frac{1}{b}\sum\limits_{i=bt+1}^{bt} \paren{\bx^*v_j-B_{j_i}^*v_{t,j_i}}=\frac{1}{b}\sum\limits_{i=bt+1}^{bt} \xi_{i,3},\\
M_{t,4} = & \frac{1}{b}\sum\limits_{i=bt+1}^{bt} \paren{w_j-w_{t,j_i}} =\frac{1}{b}\sum\limits_{i=bt+1}^{bt} \xi_{i,4},
\end{align*}
where 
\begin{align*}
\xi_{i,1}=&(\tx-T_{j_i})\et,\\
\xi_{i,2}=&\bx^*\VE-B_{j_i}^*\epsilon_{j_i},\\
\xi_{i,3}=&\bx^*v_j-B_{j_i}^*v_{t,j_i},\\
\xi_{i,4}=&w_j-w_{t,j_i}.
\end{align*}

By independence, we obtain for $t=1,\ldots, 4$:
\begin{equation}\label{M.xi}
    \EE\sbrac{M_{t,t}\otimes M_{t,t}|\mathcal{F}_n} = \frac{1}{b^2} \sum\limits_{i,i'} \EE\sbrac{\xi_{i,t}\otimes \xi_{i,t}|\mathcal{F}_n} = \frac{1}{b^2} \sum\limits_{i} \EE\sbrac{\xi_{i,t}\otimes \xi_{i,t}|\mathcal{F}_n}.
\end{equation}

Then, the first inequality follows from the fact
\begin{equation*}
 \EE\sbrac{\xi_{i,1}\otimes \xi_{i,1}|\mathcal{F}_n} \preceq   \EE\sbrac{(T_{j_i}\et)\otimes (T_{j_i}\et)|\mathcal{F}_n}\preceq \EE\sbrac{(B_{j_i}\et)^2  T_{j_i}} \preceq \kao^2\norm{\et}^2 \tx,
\end{equation*}
which implies
\begin{equation*}
  \EE\sbrac{\xi_{i,1}\otimes \xi_{i,1}} \preceq \kao^2\EE\sbrac{\norm{\et}^2} \tx. 
\end{equation*}
Now, we turn to the second inequality. From Assumption~\ref{ass:bounded}, we have 
\begin{equation}\label{yx.bd}
|y-A(\fp)(x)| \leq \widetilde{M}+\ka\norm{A(\fp)}\leq \Sigma   
\end{equation}
for $\Sigma: = 2 \max\paren{\widetilde{M},\ka \norm{A(\fp)}}$. The second inequality follows from
\begin{align*}
\EE\sbrac{\xi_{i,2}\otimes \xi_{i,2} | \FF_n}\preceq &\EE\sbrac{\paren{B_{j_i}^*\epsilon_{j_i}}\otimes \paren{B_{j_i}^*\epsilon_{j_i}} | \FF_n} \preceq
\EE\sbrac{\abs{\epsilon_{j_i}}^2T_{j_i}} \preceq \Sigma^2 \tx, 
\end{align*}
which implies
\begin{equation*}
\EE\sbrac{\xi_{i,2}\otimes \xi_{i,2}}\preceq \Sigma^2 \tx.
\end{equation*}

Next, we prove the third inequality. This is followed by observing
\begin{align*}
\EE\sbrac{\xi_{i,3}\otimes \xi_{i,3} | \FF_n}\preceq &\EE\sbrac{\paren{B_{j_i}^*v_{t,j_i}}\otimes \paren{B_{j_i}^*v_{t,j_i}} | \FF_n} \preceq
\EE\sbrac{(v_
{t,j_i})^2T_{j_i}} \preceq \ka^2 \norm{r(\ft)}^2 \tx,
\end{align*}
which implies
\begin{equation*}
\EE\sbrac{\xi_{i,3}\otimes \xi_{i,3} } \preceq \ka^2 \EE\sbrac{\norm{r(\ft)}^2} \tx  \preceq 4 \ka^2 \lip^2\EE\sbrac{\norm{\et}^2} \tx.
\end{equation*}

From the inequalities \eqref{eq:lipschitz3}, \eqref{yx.bd} We have the inequality:
\begin{align*}
\paren{A(\ft)(x_{j_i})-y_{j_i}}^2 \leq & 2\paren{A(\ft)(x_{j_i})-A(\fp)(x_{j_i})}^2+ 2\paren{A(\fp)(x_{j_i})-y_{j_i}}^2 \\
\leq & 2\paren{\frac{\ka \cdot \lip }{1-\alpha}}^2 \cdot \norm{\ft-\fp}^2+2\Sigma^2 =C^{(3)}_{\alpha, \lip} \cdot \paren{\norm{\ft-\fp}^2+1}, 
\end{align*}
where 
\[ C^{(3)}_{\alpha, \lip} =  2 \max\left\{  \paren{\frac{\ka \cdot \lip }{1-\alpha}}^2 , \Sigma^2 \right\}  \;. \] 
This yields
\begin{align*}
&\EE\sbrac{\xi_{i,4}\otimes \xi_{i,4}}\\
\preceq & \EE\sbrac{\paren{w_{t,j_i}}\otimes \paren{w_{t,j_i}} } \\
=& \EE\sbrac{(A(\ft)(x_{j_i})-y_{j_i})^2 \cdot \paren{A'(\ft)-A'(\fp)}^*S_{x_{j_i}}^*S_{x_{j_i}}\paren{A'(\ft)-A'(\fp)}}\\
\preceq & C^{(3)}_{\alpha, \lip} \EE\sbrac{\paren{\norm{\et}^2+1} \cdot \paren{A'(\ft)-A'(\fp)}^*\lx\paren{A'(\ft)-A'(\fp)}}.
\end{align*}
Then, the last inequality follows from \eqref{M.xi}.
\end{proof}

\begin{lemma}\label{Lem.Mi.Mi}
Let Assumptions~\ref{ass:bounded}, \ref{ass:kernel}, \ref{ass:Frechet}, and \ref{ass:A} hold true. For any $t\in\NN$ and $\la>0$, we have the following almost surely:
\begin{align*}
\tr\sbrac{\tp^{2u}Q_{t-j}^2\EE\paren{M_{t,1}\otimes M_{t,1}}}\leq &   \frac{\kao^2}{b}\EE\sbrac{\norm{\et}^2} \tr\sbrac{\tp^{2u}Q_{t-j}^2\tx},\\
\tr\sbrac{\tp^{2u}Q_{t-j}^2\EE\paren{M_{t,2}\otimes M_{t,2}}}\leq &  \frac{\Sigma^2}{b} \tr\sbrac{\tp^{2u}Q_{t-j}^2\tx},\\
\tr\sbrac{\tp^{2u}Q_{t-j}^2\EE\paren{M_{t,3}\otimes M_{t,3}}}\leq  & \frac{4\kao^2}{b}\EE\sbrac{\norm{\et}^2} \tr\sbrac{\tp^{2u}Q_{t-j}^2\tx},\\
\tr\sbrac{\tp^{2u}Q_{t-j}^2\EE\paren{M_{t,4}\otimes M_{t,4}}} \leq &  \frac{C^{(4)}_{\alpha, \lip}}{b} \cdot\paren{\EE\sbrac{\norm{\et}^2}+\EE\sbrac{\norm{\et}^4}} 
       \nonumber\\
&\cdot \brac{\tr\sbrac{\tp^{2u+1}Q_{t-j}^2}+\Psi_{\xx}\cdot\tr\sbrac{\tp^{2u+\frac{1}{2}}Q_{t-j}^2} +\la\Upsilon_{\xx} \cdot\norm{\tp^{2u}Q_{t-j}^2 }},
\end{align*}
where $\Sigma = 2 \max\paren{\widetilde{M},\ka \norm{A(\fp)}}$ and $C^{(4)}_{\alpha, \lip}= C^{(3)}_{\alpha, \lip} \cdot \paren{C_R^2+C_R\eld+\eld^2}$. The expectation is taken over the $b$-fold uniform distribution on $[n]$ at step $t$.
\end{lemma}
\begin{proof}
The first three identities directly follow from Lemma~\ref{Lem.Mi}. Next, we turn to the last inequality. For this, we start with the identity:
\begin{align}\label{Alx}
&\paren{A'(\ft)-A'(\fp)}^*\lx\paren{A'(\ft)-A'(\fp)}  \nonumber   \\
= & \paren{A'(\ft)-A'(\fp)}^*\lp\paren{A'(\ft)-A'(\fp)}        \nonumber \\
&+ \paren{A'(\ft)-A'(\fp)}^*\lp (\lp+\la I)^{-1}(\lx-\lp)\paren{A'(\ft)-A'(\fp)}             \nonumber \\
&+ \paren{A'(\ft)-A'(\fp)}^*\la (\lp+\la I)^{-1} (\lx-\lp)\paren{A'(\ft)-A'(\fp)}                        \nonumber \\ 
= &\paren{R_{\ft} \ip A'(\fp) -\ip A'(\fp)}^* \paren{R_{\ft} \ip A'(\fp) -\ip A'(\fp)}        \nonumber  \\
&+ \paren{R_{\ft} \ip A'(\fp) -\ip A'(\fp)}^*\ip (\lp+\la I)^{-1}(\lx-\lp)\paren{A'(\ft)-A'(\fp)}         \nonumber  \\ 
&+ \la\paren{A'(\ft)-A'(\fp)}^* (\lp+\la I)^{-1} (\lx-\lp)\paren{A'(\ft)-A'(\fp)}        \nonumber   \\
= & \bp^*(R_{\ft}- I)^2\bp + \bp^*(R_{\ft}- I)\ip (\lp+\la I)^{-1}(\lx-\lp)\paren{A'(\ft)-A'(\fp)}        \nonumber  \\
&+ \la\paren{A'(\ft)-A'(\fp)}^* (\lp+\la I)^{-1} (\lx-\lp)\paren{A'(\ft)-A'(\fp)}.
\end{align}

For the bounded operator $A$ and the trace class operator $B$, we have that $Tr(AB)\leq \norm{A}Tr(B)$. Using this fact with Assumption~\ref{ass:A}, Lemma~\ref{Lem.Mi} and \eqref{Alx} we obtain
\begin{align*}\label{tx.Mk}
& \tr\paren{\tp^{2u}Q_{t-j}^2 \EE[M_{t,4}\otimes M_{t,4}]}        \nonumber \\
\leq & \frac{C^{(3)}_{\alpha, \lip}}{b} \tr\sbrac{\tp^{2u}Q_{t-j}^2 \cdot \EE\sbrac{\paren{\norm{\et}^2+1} \cdot \paren{A'(\ft)-A'(\fp)}^*\lx\paren{A'(\ft)-A'(\fp)}}}         \nonumber  \\
\leq & \frac{C^{(3)}_{\alpha, \lip}}{b} \Bigg\{\tr\sbrac{\tp^{2u+1}Q_{t-j}^2} \cdot\EE\sbrac{\paren{\norm{\et}^2+1}\cdot\norm{R_{\ft}-I}^2}             \nonumber \\ 
&+ \tr\sbrac{\tp^{2u+\frac{1}{2}}Q_{t-j}^2} \cdot  \norm{ \ip(\lp+\la I)^{-\frac{1}{2}}} \cdot \norm{(\lp+\la I)^{-\frac{1}{2}} (\lx-\lp)} \nonumber \\
&\cdot \EE\sbrac{\paren{\norm{\et}^2+1} \cdot \norm{R_{\ft}-I} \cdot \norm{A'(\ft)-A'(\fp)}}      \nonumber  \\ 
& +\la \cdot \norm{\tp^{2u}Q_{t-j}^2 }\cdot\tr\sbrac{(\lp+\la I)^{-1} (\lx-\lp)} \cdot\EE\sbrac{\paren{\norm{\et}^2+1}\cdot\norm{A'(\ft)-A'(\fp)}^2} \Bigg\}   \nonumber  \\ 
\leq & \frac{C^{(3)}_{\alpha, \lip}}{b} \cdot\paren{C_R^2+C_R\eld+\eld^2} \cdot\paren{\EE\sbrac{\norm{\et}^2}+\EE\sbrac{\norm{\et}^4}} 
       \nonumber\\
&\cdot \brac{\tr\sbrac{\tp^{2u+1}Q_{t-j}^2}+\Psi_{\xx}\cdot\tr\sbrac{\tp^{2u+\frac{1}{2}}Q_{t-j}^2} +\la\Upsilon_{\xx} \cdot\norm{\tp^{2u}Q_{t-j}^2 }}  \nonumber  \\ 
= & \frac{C^{(4)}_{\alpha, \lip}}{b} \cdot\paren{\EE\sbrac{\norm{\et}^2}+\EE\sbrac{\norm{\et}^4}} 
       \nonumber\\
&\cdot \brac{\tr\sbrac{\tp^{2u+1}Q_{t-j}^2}+\Psi_{\xx}\cdot\tr\sbrac{\tp^{2u+\frac{1}{2}}Q_{t-j}^2} +\la\Upsilon_{\xx} \cdot\norm{\tp^{2u}Q_{t-j}^2 }},
\end{align*}
where $C^{(4)}_{\alpha, \lip}= C^{(3)}_{\alpha, \lip} \cdot \paren{C_R^2+C_R\eld+\eld^2}$.
This proves the lemma.
\end{proof}


\section{Probabilistic Bounds}
\label{Sec:Prob.bound}


Here, we present the standard perturbation inequalities in learning theory which measures the effect of random sampling in the probabilistic sense. 
The following two lemmas can be proved using the arguments given in Step 2.1. of \cite[Thm.~4]{Caponnetto}.

\begin{lemma}\label{main.bound}
Suppose Assumptions~\ref{ass:true}, \ref{ass:noise}, \ref{ass:kernel} hold true, then for~$n \in \NN$ and~$0<\delta<1$, the following estimates hold with the confidence~$1-\delta$,
\begin{equation*}
\Upsilon_{\xx}= \Upsilon_{\xx}(\la): =\tr\sbrac{(\lp+\la I)^{-1}(\lp-\lx)} \leq 2\paren{\frac{\ka^2}{n\la}+\sqrt{\frac{\ka^2\mathcal{N}(\la)}{n\la}}}\log\left(\frac{6}{\delta}\right),
\end{equation*}

\begin{align*}
\Psi_{\xx} = &\Psi_{\xx}(\la) := \norm{(\lp+\la I)^{-1/2}(\lx-\lp)}_{HS}\leq 2\left(\frac{\ka^2}{n\sqrt{\la}}+\sqrt{\frac{\ka^2\mathcal{N}(\la)}{n}}\right)\log\left(\frac{6}{\delta}\right),\\
\intertext{and}
\Theta_{\zz} = &\Theta_{\zz}(\la) := \norm{(\lp+\la I)^{-1/2}\sx^*(\yy-\sx A(\fp))}_{\HH} \leq 2\paren{\frac{\ka M}{n\sqrt{\la}}+\sqrt{\frac{\Sigma^2\mathcal{N}(\la)}{n}}}\log\left(\frac{6}{\delta}\right).
\end{align*}
\end{lemma}

In the following lemma, we simplify the probabilistic estimate of Lemma \ref{main.bound} under Assumption~\ref{ass:noise} and a condition on the regularization parameter~$\la$ in terms of sample size~$n$.

\begin{lemma}\label{cor:prop-bounds}
Suppose Assumption~\ref{ass:true}, \ref{ass:noise}, \ref{ass:kernel}, \ref{ass:noise} and the condition~$n^{-\frac{1}{1+\nu}}\leq \la \leq 1$ hold true. Let~$0\leq s \leq 1$, then for~$n \in \NN$ and~$0<\delta<1$, the following estimates hold with the confidence~$1-\delta$,

\begin{align*}
\Upsilon_{\xx}  \leq &C_{\ka}\log\left(\frac{6}{\delta}\right),\\
\Psi_{\xx} \leq &  C_{\ka}\sqrt{\la}\log\paren{\frac{6}{\delta}},\\
\Psi_{\xx}  \leq & C_{\ka}\frac{1}{\sqrt{n\la^{\nu}}}\log\left(\frac{6}{\delta}\right),\\
\Theta_{\zz}  \leq &   C_{\ka, M, \Sigma}\frac{1}{\sqrt{n\la^{\nu}}}\log\left(\frac{6}{\delta}\right),\\
\intertext{and}
\Xi^s\leq & \paren{C_{\ka}\log\paren{\frac{6}{\delta}}}^{2s}.
\end{align*}
where $C_{\ka}=2 (\ka^2+\ka C_\nu)$ and $C_{\ka, M, \Sigma}=2\paren{\ka M+\Sigma C_\nu}$. 
\end{lemma}

\begin{proof}
From Assumption~\ref{ass:eff.bound} we obtain,
\begin{equation}\label{domino.bd}
 \sqrt{\frac{\mathcal{N}(\la)}{n\la}}\leq \frac{C_\nu}{\sqrt{n\la^{1+\nu}}} 
\end{equation}
and for $\la\leq 1$
\begin{equation}\label{domino.bd.1}
\frac{1}{n\la}\leq \frac{1}{n\la^{1+\nu}}= \frac{1}{\sqrt{n\la^{1+\nu}}}\frac{1}{\sqrt{n\la^{1+\nu}}}\leq\frac{1}{\sqrt{n\la^{1+\nu}}}.
\end{equation}
Now using \eqref{domino.bd} and \eqref{domino.bd.1} in Lemma~\ref{main.bound} we get the following estimates with probability~$1-\delta$,

\begin{equation*}\label{Upsilon.bound}
\Upsilon_{\xx}  \leq 2 (\ka^2+\ka C_\nu)\log\left(\frac{6}{\delta}\right),
\end{equation*}

\begin{equation*}\label{Theta.bound}
\Theta_{\zz}  \leq 2\paren{\ka M+\Sigma C_\nu}\frac{1}{\sqrt{n\la^{\nu}}}\log\left(\frac{6}{\delta}\right)
\end{equation*}
and
\begin{equation*}\label{Psi.bound}
\Psi_{\xx}  \leq 2 (\ka^2+\ka C_\nu)\frac{1}{\sqrt{n\la^{\nu}}}\log\left(\frac{6}{\delta}\right) \leq 2 (\ka^2+\ka C_\nu)\sqrt{\la}\log\left(\frac{6}{\delta}\right).
\end{equation*}

Using \cite[Prop.~A.2]{Blanchard2019}, we obtain with confidence $1-\frac{\delta}{3}$:
\begin{equation*}
\Xi^s =\norm{(\lx+\la I)^{-s}(\lp+\la I)^s}\leq  \paren{\frac{\Psi_{\xx}}{\sqrt{\la}}+1}^{2s} \leq \paren{C_{\ka}\log\paren{\frac{6}{\delta}}}^{2s}.    
\end{equation*}
\end{proof}



\checknbnotes

\end{document}